\documentclass{article}
\pdfoutput=1

\usepackage[preprint]{neurips_2022}

\usepackage{textcomp}

\usepackage{algorithm}
\usepackage{algpseudocode}

\usepackage{amsthm}
\usepackage{mathtools}
\mathtoolsset{showonlyrefs}

\newcommand{\norm}[1]{\left\lVert#1\right\rVert}

\newcommand{\Lcal}[0]{\mathcal{L}}

\newcommand{\Ocal}[0]{\mathcal{O}}

\newcommand{\Zcal}[0]{\mathcal{Z}}

\newcommand{\expect}{\operatorname{\mathbb{E}}}
\newcommand{\Var}{\operatorname{Var}}

\newcommand{\VAR}[1]{\Var\left[#1\right]}

\newcommand{\E}[1]{\expect\left[#1\right]}
\newcommand{\EE}[2]{\expect_{#1}\left[#2\right]}

\newcommand{\inner}[2]{\langle #1, #2 \rangle}

\newtheorem{lemma}{Lemma}
\newtheorem{theorem}{Theorem}
\newtheorem{assumption}{Assumption}

\newtheorem{corollary}{Corollary}
\newtheorem{example}{Example}

\usepackage{amsmath,amsfonts,bm}

\def\eqref#1{equation~\ref{#1}}

\def\ceil#1{\left \lceil #1 \right \rceil}
\def\floor#1{\left \lfloor #1 \right \rfloor}
\def\1{\bm{1}}

\def\vDelta{{\bm{\Delta}}}
\def\vtheta{{\bm{\theta}}}
\def\vxi{{\bm{\xi}}}

\def\vx{{\bm{x}}}
\def\vy{{\bm{y}}}
\def\vz{{\bm{z}}}

\DeclareMathAlphabet{\mathsfit}{\encodingdefault}{\sfdefault}{m}{sl}
\SetMathAlphabet{\mathsfit}{bold}{\encodingdefault}{\sfdefault}{bx}{n}

\begin{document}

\title{A General Theory for Federated Optimization with Asynchronous and Heterogeneous Clients Updates}

\author{%
	Yann Fraboni\\
	Universit\'e C\^{o}te d'Azur, Inria Sophia Antipolis,\\
	Epione Research Group, France\\
	Accenture Labs, Sophia Antipolis, France\\
	\And
	Richard Vidal \\
	Accenture Labs, Sophia Antipolis, France\\ 
	\And
	Laetitia Kameni \\
	Accenture Labs, Sophia Antipolis, France\\
	\And
	Marco Lorenzi \\
	Universit\'e C\^{o}te d'Azur, Inria Sophia Antipolis,\\
	Epione Research Group, France\\
}


\maketitle

\begin{abstract}
	We propose a novel framework to study asynchronous federated learning optimization with delays in gradient updates. Our theoretical framework extends the standard \textsc{FedAvg} aggregation scheme by introducing stochastic aggregation weights to represent the variability of the clients update time, due for example to heterogeneous hardware capabilities. Our formalism applies to the general federated setting where clients have heterogeneous datasets and perform at least one step of stochastic gradient descent (SGD). We demonstrate convergence for such a scheme and provide sufficient conditions for the related minimum to be the optimum of the federated problem. We show that our general framework applies to existing optimization schemes including centralized learning, \textsc{FedAvg}, asynchronous \textsc{FedAvg}, and \textsc{FedBuff}. The theory here provided allows drawing meaningful guidelines for designing a federated learning experiment in heterogeneous conditions. In particular, we develop in this work \textsc{FedFix}, a novel extension of \textsc{FedAvg} enabling efficient asynchronous federated training while preserving the convergence stability of synchronous aggregation. We empirically demonstrate our theory on a series of experiments showing that asynchronous \textsc{FedAvg} leads to fast convergence at the expense of stability, and we finally demonstrate  the improvements of \textsc{FedFix} over synchronous and asynchronous \textsc{FedAvg}.
\end{abstract}

\section{Introduction}

Federated learning (FL) is a training paradigm enabling different clients to jointly learn a global model without sharing their respective data. Federated learning is a generalization of distributed learning (DL), which was first introduced to optimize a given model in star-shaped networks composed of a server communicating with computing machines \citep{Bertsekas1989, Nedi2001, SlowLearnersAreFast}. 
In DL, the server owns the dataset and distributes it across machines. 
At every optimization round, the machines return the estimated gradients, and the server aggregates them to perform an SGD step. DL was later extended to account for SGD, and FL extends DL to enable optimization without sharing data between clients. Typical federated training schemes are based on the averaging of clients model parameters optimized locally by each client, such as in \textsc{FedAvg} \citep{FedAvg}, where at every optimization round clients perform a fixed amount of stochastic gradient descent (SGD) steps initialized with the current global model parameters, and subsequently return the optimized parameters to the server. The server computes the new global model as the average of the clients updates weighted by their respective data ratio.

A key methodological difference between the optimization problem solved in FL and the one of DL lies in the assumption of potentially non independent and identically distributed (iid) data instances \citep{OpenProblems, YangFML}.
Proving convergence in the non-iid setup is more challenging, and in some settings, \textsc{FedAvg} has been shown to converge to a sub-optimum, e.g. when each client performs a different amount of local work \citep{FedNova}, or when clients are not sampled in expectation according to their importance \citep{PowerOfChoice}. 

A major drawback of \textsc{FedAvg} concerns the time needed to complete an optimization round, as the server must wait for all the clients to perform their local work to \textit{synchronize} their update and create a new global model.
As a consequence, due to the potential heterogeneity of the hardware across clients, the time for an optimization round is conditioned to the one of the slowest update, while the fastest clients stay idle once they have sent their updates to the server. 
To address these limitations, asynchronous FL has been proposed to take full advantage of the clients computation capabilities \citep{AFLsurvey, Hogwild2, KoloskovaGossip, TamingTheWild}. 
In the asynchronous setting, whenever the server receives a client's contribution, it creates a new global model and sends it back to the client.
In this way, clients are never idle and always perform local work on a different version of the global model.
While asynchronous FL has been investigated in the iid case \citep{ErrorFeedbackFramework}, a unified theoretical and practical investigation in the non-iid scenario is currently missing.

This work introduces a novel theoretical framework for asynchronous FL based on the generalization of the aggregation scheme of \textsc{FedAvg}, where asynchronicity is modeled as a stochastic process affecting clients' contribution at a given federated aggregation step.
More specifically, our framework is based on a stochastic formulation of FL, where clients are given stochastic aggregation weights dependent on their effectiveness in returning an update. 
Based on this formulation, we provide sufficient conditions for asynchronous FL to converge, and we subsequently give sufficient conditions for convergence to the FL optimum of the associated synchronous FL problem. 
Our conditions depend on the clients computation time (which can be eventually estimated by the server), and are independent from the clients data heterogeneity, which is usually unknown to the server.

With asynchronous FL, the server only waits for one client contribution to create the new global. As a result, optimization rounds are potentially faster even though the new global improves only for the participating client at the detriment of the other ones. This aspect may affect the stability of asynchronous \textsc{FedAvg} as compared to synchronous \textsc{FedAvg} and, as we demonstrate in this work, even diverge in some cases. To tackle this issue, we propose \textsc{FedFix}, a robust asynchronous FL scheme, where new global models are created with all the clients contributions received after a fixed amount of time.
We prove the convergence of \textsc{FedFix} and verify experimentally that it outperforms standard asynchronous \textsc{FedAvg} in the considered experimental scenarios.

The paper is structured as follows. 
In Section \ref{sec:background}, we introduce our aggregation scheme and the close-form of its aggregation weights in function of the clients computation capabilities and the considered FL optimization routine.
Based on our aggregation scheme, in Section \ref{sec:convergence}, we provide convergence guarantees, and we give sufficient conditions for the learning procedure to converge to the optimum of the FL optimization problem. 
In Section \ref{sec:applications}, we apply our theoretical framework to synchronous and asynchronous \textsc{FedAvg}, and show that our work extends current state-of-the-art approaches to asynchronous optimization in FL. Finally, in Section \ref{sec:experiments}, we demonstrate experimentally our theoretical results.

\section{Background}\label{sec:background}

We define here the formalism  required by the theory that will be introduced in the following sections.
We first introduce in Section \ref{subsec:background_FedAvg} the FL optimization problem, and we adapt it in section \ref{subsec:formalize_delays} to account for delays in client contributions. We then generalize in Section \ref{subsec:agg_schem} the \textsc{FedAvg} aggregation scheme to account for contributions delays. 
In Section \ref{subsec:FL_one_SGD}, we introduce the notion of virtual global models as a direct generalization of gradient descent, and introduce in Section \ref{subsec:sequence_opt_problems} the final asynchronous FL optimization problem. Finally, we introduce in Section \ref{subsec:heterogeneity_across_clients} a formalization of the concept of data heterogeneity across clients.

\subsection{Federated Optimization Problem}\label{subsec:background_FedAvg}

We have $M$ participants owning $n_i$ data points $\{\vz_{k, i}\}_{k=1}^{n_i}$ independently sampled from a fixed 
unknown distribution over a sample space $\{\mathcal{Z}_i\}_{i=1}^M$.
We have $\vz_{k, i} = (\vx_{k, i}, \vy_{k, i})$ for supervised learning, where $\vx_{k, i}$ is the input of the statistical model,  and $\vy_{k, i}$ its desired target, while we denote $\vz_{k, i} = \vx_{k, i}$ for unsupervised learning.
Each client optimizes the model's parameters $\vtheta$ based on the estimated local loss $l(\vtheta, \vz_{k, i})$. 
The aim of FL is solving a distributed optimization problem associated with the averaged loss across clients
\begin{equation}
\Lcal(\vtheta) 
\coloneqq \EE{z \sim \hat{\mathcal{Z}}}{l(\vtheta, \vz)}
= \frac{1}{\sum_{i=1}^M n_i}\sum_{i=1}^M \sum_{k=1}^{n_i} l(\vtheta, \vz_{k, i}) \label{eq:original_problem}
,
\end{equation}
where the expectation is taken with respect to the sample distribution $\hat{\mathcal{Z}}$ across the $M$ participating clients.
We consider a general form of the federated loss of equation (\refeq{eq:original_problem}) where 
clients local losses are weighted by an associated parameter $p_i$ such that $\sum_{i=1}^{n} p_i = 1$, i.e.
\begin{equation}
\label{eq:original_problem_general}
\Lcal(\vtheta) = \sum_{i=1}^Mp_i \Lcal_i(\vtheta)
\text{ s.t. }
\Lcal_i(\vtheta) = \frac{1}{n_i}\sum_{k=1}^{n_i} l(\vtheta, \vz_{k, i})
.
\end{equation}
The weight $p_i$ can be interpreted as the importance given by the server to client $i$ in the federated optimization problem. While any combination of $\{p_i\}$ is possible, we note that in typical FL formulations, either (a) every client has equal importance, i.e. $p_i = 1 / M$, or (b) every data point is equally important, i.e. $p_i = n_i / \sum_{i=1}^Mn_i$. 



\subsection{Asynchronicity in Clients Updates}\label{subsec:formalize_delays}

An optimization round starts at time $t^n$ with global model $\vtheta^n$, finishes at time $t^{n+1}$ with the new global model $\vtheta^{n+1}$, and takes $\Delta t^{n} = t^{n+1} - t^n$ time to complete. No assumptions are made on $\Delta t^{n}$, which can be a random variable, and we set for convenience $t^0 = 0$. In this section, we introduce the random variables needed to develop in Section \ref{subsec:agg_schem} the server aggregation scheme connecting two consecutive global models $\vtheta^n$ and $\vtheta^{n+1}$.

We define the random variable $T_i$ representing the update time needed for client $i$ to perform its local work and send it to the server for aggregation. $T_i$ depends on the client computation and communication hardware, and is assumed to be independent from the current optimization round $n$. 
If the server sets the FL round time to $\Delta t^{n} = \max_i T_i$, the aggregation is performed by waiting for the contribution of every client, and we retrieve the standard client-server communication scheme of synchronous \textsc{FedAvg}.

With asynchronous \textsc{FedAvg}, we need to relate $T_i$ to the server aggregation time $\Delta t^n$. 
We introduce $\rho_i(n)$ the index of the most recent global model received by client $i$ at optimization round $n$ and, by construction, we have $0 \le \rho_i(n) \le n$.
We define by 
\begin{equation}
	T_i^n 
	\coloneqq T_i - (t^n - t^{\rho_i(n)})
\end{equation}the remaining time at optimization round $n$ needed by client $i$ to complete its local work. 

Comparing $T_i^n$ with $\Delta t^n$ indicates whether a client is participating to the optimization round or not, through the stochastic event $\mathbb{I}( T_i^n \le \Delta t^{n})$. When $\mathbb{I}( T_i^n \le \Delta t^n) = 1$, the local work of client $i$ is used to create the new global model $\vtheta^{n+1}$, while client $i$ does not contribute when $\mathbb{I}( T_i^n \le \Delta t^n) = 0$. With synchronous \textsc{FedAvg}, we retrieve $\mathbb{I}( T_i^n \le \Delta t^n) = \mathbb{I}( T_i \le \max_i T_i) = 1$ for every client.

\begin{figure}
	\includegraphics[width=\linewidth]{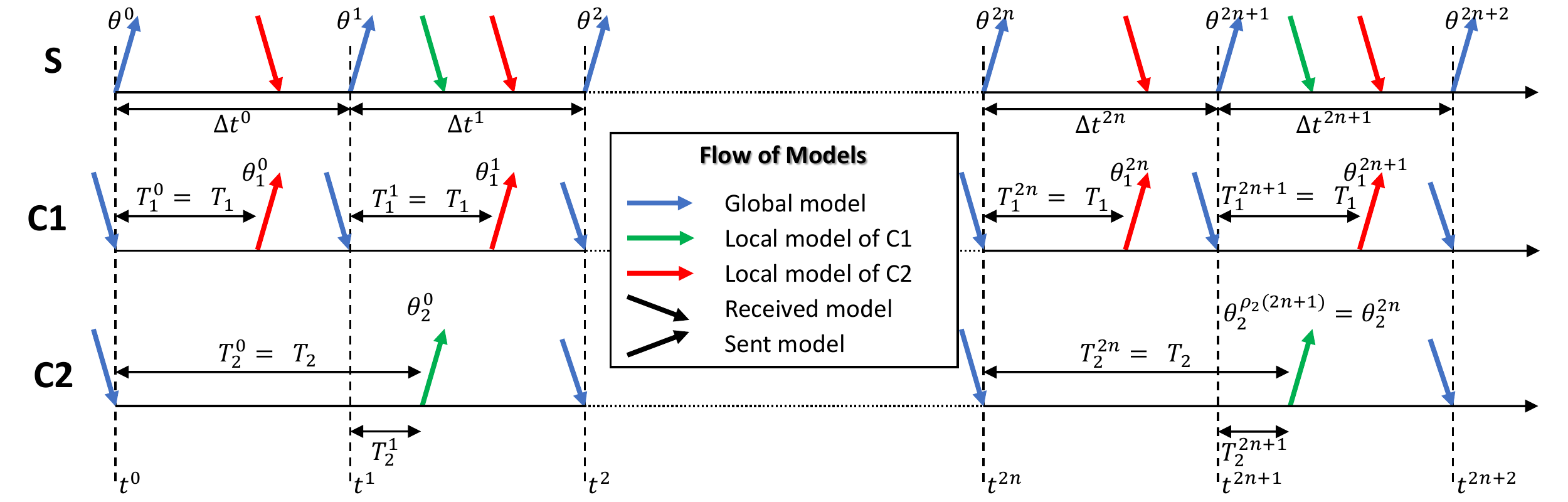}
	\caption{Illustration of the time notations introduced in Section \ref{subsec:formalize_delays} with $M=2$ clients. The frequency of the updates of Client 1 (C1) is twice the one of Client 2 (C2). If the server (S) creates the new global model after every fixed waiting time ($\Delta t^n = \Delta t$), C1 contributes at every optimization round, while C2 contributes once every two rounds. This aggregation policy define the federated learning strategy \textsc{FedFix} (Section \ref{subsec:FedFix})}
	\label{fig:illustration}
\end{figure}

Figure \ref{fig:illustration} illustrates the notations described in this section in a FL process with $M=2$ clients.

\subsection{Server Aggregation Scheme}\label{subsec:agg_schem}

We consider $\vDelta_i(n)$ the contribution of client $i$ received by the server at optimization round $n$.
In the rest of this work, we consider that clients perform $K$ steps of SGD on the model they receive from the server. By calling their trained model $\vtheta_i^{(n, k)}$ after $k$ SGD, we can rewrite clients contribution for \textsc{FedAvg} as $\vDelta_i(n) \coloneqq \vtheta_i^{(n, K)} - \vtheta^n$, and the \textsc{FedAvg} aggregation scheme as
\begin{equation}
\label{eq:FedAvg_server_aggregation}
\vtheta^{n+1} 
\coloneqq \vtheta^n + \sum_{i=1}^M p_i \vDelta_i(n)
.
\end{equation}
With \textsc{FedAvg}, the server waits for every client to send its contribution $\vDelta_i(n)$ to create the new global model. To allow for partial computation within the server aggregation scheme, we introduce the aggregation weight $d_i(n)$ corresponding to the weight given by the server to client $i$ at optimization round $n$. We can then define the stochastic aggregation weight $\omega_i(n)$ given to client $i$ at optimization step $n$ as
\begin{equation}
\label{eq:agg_weights}
\omega_i(n) 
\coloneqq
\mathbb{I}( T_i^n \le \Delta t^{n}) d_i(n) ,
\end{equation}
with $\omega_i(n)= d_i(n)$ if client $i$ updated its work at optimization round $n$ and $\omega_i(n)= 0$ otherwise. 
In the general setting, client $i$ receives $\vtheta^{\rho_i(n) }$ and its contribution is $\vDelta_i(\rho_i(n)) = \vtheta_i^{(\rho_i(n), K) }- \vtheta^{\rho_i(n)}$. 
By weighting each delayed contribution $\vDelta_i(\rho_i(n))$ with its stochastic aggregation weight $\omega_i(n)$, we propose the following aggregation scheme  
\begin{equation}
\label{eq:aggreg_SCAFFOLD_general}
\vtheta^{n+1}
\coloneqq \vtheta^n + \eta_g \sum_{i =1}^M \omega_i(n)\vDelta_i(\rho_i(n)),
\end{equation}
where $\eta_g$ is a global learning rate that the server can use to mitigate the disparity in clients contributions \citep{reddi2021adaptive, SCAFFOLD, Wang2020SlowMo}. 
Equation (\ref{eq:aggreg_SCAFFOLD_general}) generalizes FedAvg aggregation scheme (\refeq{eq:FedAvg_server_aggregation}) ($\eta_g =1$ and $\Delta t^n = \max_i T_i$), and the one of \cite{OnTheImpact} based on client sampling.



We introduce with Algorithm \ref{alg:FL_with_delayed_gradients} the implementation of the optimization schemes satisfying aggregation scheme (\refeq{eq:aggreg_SCAFFOLD_general}) with stochastic aggregation weights satisfying equation (\refeq{eq:agg_weights}).

\begin{algorithm}
	\caption{Asynchronous Federated Learning based on equation (\ref{eq:aggreg_SCAFFOLD_general})
	}\label{alg:FL_with_delayed_gradients}
	\begin{algorithmic}[1]
		\Require server learning rate $\eta_g$, aggregation weights $\{d_i(n)\}$, number of SGD $K$, learning rate $\eta_l$, batch size $B$, aggregation time policy $\Delta t^n$.
		\State The server sends to the $M$ clients the learning parameters ($K$, $\eta_l$, $B$) and the initial global model $\vtheta^0$.
		
		\For{$n \in \{0, ..., N-1\}$}
		
		\State Clients in $S_n = \{i: T_i^n \le \Delta t^n\}$ send their contribution $\vDelta_i(\rho_i(n)) = \vtheta_i^{\rho_i(n)+1} -\vtheta^{\rho_i(n)}$ to the server.
		
		\State The server creates the new global model $\vtheta^{n+1} = \vtheta^n + \eta_g \sum_{ i \in S_n } d_i(n) \vDelta_i(\rho_i(n))$, equation (\refeq{eq:aggreg_SCAFFOLD_general}).
		
		\State The global model  $\vtheta^{n+1}$ is sent back to the clients in $S_n$.
		
		\EndFor
	\end{algorithmic}
\end{algorithm}

\subsection{Expressing FL as cumulative GD steps}\label{subsec:FL_one_SGD}
To obtain the tightest possible convergence bound, we consider a convergence framework similar to the one of \cite{OnTheConvergence} and \cite{Khaled2020}. We introduced the aggregation rule for the server global models $\{\vtheta^n\}$ in Section \ref{subsec:agg_schem}, and we generalize it in this section by introducing the virtual sequence of global models $\vtheta^{n, k}$. This sequence corresponds to the \textit{virtual} global model that would be obtained with the clients contribution at optimization round $n$ computed on $k\le K$ SGD, i.e.
\begin{equation}
\vtheta^{n, k}
\coloneqq \vtheta^n  + \eta_g \sum_{ i = 1 }^M \omega_i(n) \left[\vtheta_i^{(\rho_i(n), k)} - \vtheta^{\rho_i(n)}\right]
.
\end{equation}
We retrieve $\vtheta^{n, 0} = \vtheta^n $ and $\vtheta^{n, K} = \vtheta^{n+1, 0} = \vtheta^{n+1}$. The server has not access to  $\vtheta^{n, k}$ when $k\neq 0$ or $k\neq K$. Hence the name virtual for the model $\vtheta^{n , k}$.

The difference between two consecutive global models in our virtual sequence depends on the sum of the differences between local models $\vtheta_i^{\rho_i(n), k+1} - \vtheta_i^{\rho_i(n), k} = - \eta_l \nabla \Lcal_i(\vtheta_i^{\rho_i(n), k}, \vxi_i)$, where $\vxi_i$ is a random batch of data samples of client $i$. Hence, we can rewrite the aggregation process as a GD step with
\begin{equation}
\vtheta^{n, k+1} =
\vtheta^{n, k} - \eta_g \eta_l \sum_{ i = 1 }^M \omega_i(n) \nabla \Lcal_i(\vtheta_i^{\rho_i(n), k}, \vxi_i)
.
\end{equation}

\subsection{Asynchronous FL as a Sequence of Optimization Problems}\label{subsec:sequence_opt_problems}

For the rest of this work, we define $q_i(n) \coloneqq \E{\omega_i(n)}$, the expected aggregation weight of client $i$ at optimization round $n$. 
No assumption is made on $q_i(n)$ which can vary across optimization rounds. The expected clients contribution $\sum_{ i = 1 }^M q_i(n)\vDelta_i(n)$ help minimizing the optimization problem $\Lcal^n$ defined as
\begin{equation}
\label{eq:surrogate_loss_fucntion}
\Lcal^n(\vtheta) 
\coloneqq
\sum _{i=1}^M q_i(n) \Lcal_i(\vtheta)
.
\end{equation}
We denote by $\bar{\vtheta}^n$ the optimum of $\Lcal^n$ and by $\vtheta^*$ the optimum of the optimization problem $\Lcal$ defined in equation (\refeq{eq:original_problem_general}). Finally, we define by $q_i = \frac{1}{N}\sum_{ n = 0 }^{N-1}q_i(n)$ the expected importance given to client $i$ over the $N$ server aggregations during the FL process,
and by $\tilde{q}_i(n)$ the normalized expected importance $\tilde{q}_i(n) = q_i(n) / (\sum_{ i = 1 }^M q_i(n))$. 
We define by $\bar{\Lcal}$ the associated optimization problem 
\begin{equation}
\label{eq:surrogate_problem}
\bar{\Lcal}(\vtheta) 
\coloneqq
\sum _{i=1}^M q_i \Lcal_i(\vtheta)
=
\frac{1}{N}\sum _{ n = 0}^{N-1} \Lcal^n(\vtheta)
,
\end{equation}
and we denote by $\bar{\vtheta}$ the associated optimum.

Finally, we introduce the following expected convergence residual, which quantifies the variance at the optimum in function of the relative clients importance $q_i(n)$ 
\begin{equation}
\Sigma
\coloneqq \sum_{ i = 1 }^M q_i \EE{\vxi_i}{\norm{\nabla \Lcal_i(\bar{\vtheta}, \vxi_i)}^2}
.
\end{equation}
The convergence guarantees provided in this work (Section \ref{sec:convergence}) are proportional to the expected convergence residual.
$\Sigma$ is positive and null only when clients have the same loss function and perform GD steps for local optimization.

\subsection{Formalizing Heterogeneity across Clients}\label{subsec:heterogeneity_across_clients}

We assume the existence of $J \le M$ different clients feature spaces $\Zcal_i$ and, without loss of generality, assume that the first $J$ clients feature spaces are different. 
This formalism allows us to represent the heterogeneity of data distribution across clients. 
In DL problems, we have $J <M$ when the same dataset split is accessible to many clients. 
When clients share the same distribution, we assume that their optimization problem is equivalent. In this case, we call $F_j(\vtheta)$ their loss function with optimum $\vtheta_j^*$.
The federated problem of equation (\refeq{eq:original_problem_general}) can thus be formalized with respect to the discrepancy between the clients feature spaces $\Zcal_i$. To this end, we define $Q_j$ the set of clients with the same feature space of client $j$, i.e. $Q_j \coloneqq \{i: \Zcal_i = \Zcal_j\}$. Each feature space as thus importance $r_j = \sum_{ i \in Q_j } p_i$, and expected importance $s_j(n) = \sum_{ i \in Q_j } q_i(n)$ such that
\begin{equation}
\Lcal (\vtheta) 
= \sum_{j=1}^J r_j F_j (\vtheta)
\text{ and }
\Lcal^n (\vtheta)
= \sum_{j=1}^J s_j(n) F_j(\vtheta)
.
\end{equation}
As for $\tilde{q}_i(n)$, we define $\tilde{s}_j(n) = s_j(n)/ \sum_{ i = 1 }^M s_j(n)$. 

In Table \ref{tab:aggregation_weights}, we summarize the different weights used to adapt the federated optimization problem to account respectively for heterogeneity in clients importance and data distributions across rounds. 


\begin{table}
	\begin{center}
		\begin{tabular}{ | c || c | c |  }
			\hline
			& Client $i$ & Sample distribution $j$ \\
			\hline
			Importance   
			& $p_i$ 
			& $r_j$
			\\
			Stochastic aggregation weight 
			&   $\omega_i(n)  $
			& -
			\\
			Aggregation weight 
			& $d_i(n)$ 
			& -
			\\
			Expected agg. weight 
			& $q_i(n)$ 
			& $s_j(n)$
			\\
			Normalized expected agg. weight 
			& $\tilde{q}_i(n)$ 
			& $\tilde{s}_j(n)$
			\\
			Expected agg. weight over $N$ rounds 
			& $q_i$ 
			& $s_j$
			\\
			\hline
		\end{tabular}
	\end{center}
\caption{The different weights used to account for the importance of clients or data distributions at every optimization round and during the full FL process.}
\label{tab:aggregation_weights}
\end{table}

\section{Convergence of Federated Problem (\ref{eq:original_problem_general})}\label{sec:convergence}

In this section, we prove the convergence of the optimization based on the stochastic aggregation scheme defined in equation (\ref{eq:aggreg_SCAFFOLD_general}), with implementation given in Algorithm \ref{alg:FL_with_delayed_gradients}. 
We first introduce in Section \ref{subsec:assumptions} the necessary assumptions and then prove with Theorem \ref{theo:convergence_convex} the convergence of the sequence of optimized models (Section \ref{subsec:convergence_convex}).
We show in Section \ref{subsec:sufficient_conditions} the implications of Theorem \ref{theo:convergence_convex} on the convergence of the federated problem (\ref{eq:original_problem_general}), and propose sufficient conditions for the learnt model to be the associated optimum. Finally, with two additional assumptions, we propose in Section \ref{subsec:relaxed_sufficient} simpler and practical sufficient conditions for FL convergence to the optimum of the federated problem (\ref{eq:original_problem_general}).

\subsection{Assumptions}\label{subsec:assumptions}

We make the following assumptions regarding the Lipschitz smoothness and convexity of the clients local loss functions (Assumption \ref{ass:smoothness} and  \ref{ass:strong_convexity}), 
unbiased gradients estimators (Assumption \ref{ass:unbiased}), finite answering time for the clients (Assumption \ref{ass:answering_time}), and the clients aggregation weights (Assumption \ref{ass:clients_covariance}). 
Assumption \ref{ass:unbiased} \citep{Khaled2020} considers unbiased gradient estimators without assuming bounded variance, giving in turn more interpretable convergence bounds.
Assumption \ref{ass:clients_covariance} states that the covariance between two aggregation weights can be expressed as the product of their expected aggregation weight up to a positive multiplicative factor $\alpha$. We show in Section \ref{sec:applications} that Assumption \ref{ass:clients_covariance} is not limiting as it is satisfied by all the standard FL optimization schemes considered in this work.

\begin{assumption}[Smoothness]\label{ass:smoothness}
    Clients local objective functions are $L$-Lipschitz smooth, that is, $\forall i \in \{1, ..., n\},\ \norm{\nabla \Lcal_i(\vx) - \nabla \Lcal_i(\vy)} \le L \norm{\vx - \vy}$.
\end{assumption}

\begin{assumption}[Convexity]\label{ass:strong_convexity}
	Clients local objective functions are convex.
\end{assumption}

\begin{assumption}[Unbiased Gradient]\label{ass:unbiased}
	Every client stochastic gradient $g_i(\vx) = \nabla \Lcal_i(\vx, \vz_i)$ of a model with parameters $\vx$ evaluated on batch $\vz_i$ is an unbiased estimator of the local gradient, i.e. $\EE{\vz_i}{g_i(\vx)} = \nabla \Lcal_i(\vx)$.
\end{assumption}

\begin{assumption}[Finite Answering Time]\label{ass:answering_time}
	The server receives a client local work in at most $\tau \coloneqq \max_{i, n} (n - \rho_i(n))$ optimization steps, which satisfy $\mathbb{P}( \tau < \infty)  = 1$.
\end{assumption}

\begin{assumption}\label{ass:clients_covariance}
	There exists $\alpha \in (0, 1)$ such that $\E{\omega_i(n)\omega_j(n)} = \alpha q_i(n) q_j(n) $.
\end{assumption}

\subsection{Convergence of Algorithm \ref{alg:FL_with_delayed_gradients}}\label{subsec:convergence_convex}

We first prove with Theorem \ref{theo:convergence_convex} the convergence of Algorithm \ref{alg:FL_with_delayed_gradients}. 

\begin{theorem}\label{theo:convergence_convex}
	Under Assumptions \ref{ass:smoothness} to \ref{ass:clients_covariance}, 
	with $\eta_l \le 1/48KL
	\min \left(1, 1/ {3 \rho^{2} \eta_g (\tau+1)}\right)$,
	we obtain the following convergence bound:
	\begin{align}
	\frac{1}{N}\sum_{ n = 0 }^{N-1}\frac{1}{K}\sum_{ k = 0 }^{K-1} \left[ \E{\Lcal^n(\vtheta^{n, k})} - \Lcal^n(\bar{\vtheta}^n) \right]
	& \le R(\{\Lcal^n\})
	+ \epsilon_F
	+ \epsilon_K
	+ \epsilon_\alpha
	+ \epsilon_\beta
	,
	\end{align}
	where 
	\begin{align}
	R(\{\Lcal^n\})
	= \frac{1}{N}\sum_{ n = 0 }^{N-1} \left[ \Lcal^n(\bar{\vtheta}) - \Lcal^n(\bar{\vtheta}^n)\right]
	,
	&&
	\epsilon_{F}
    = \frac{1}{\tilde{\eta} K N}\norm{\vtheta^{0} - \bar{\vtheta}}^2
    ,
	\end{align}
	\begin{align}
	    \epsilon_K
		= \Ocal \left(\eta_l^2 (K-1)^2 \left[ R(\{\Lcal^n\}) + \Sigma_1 \right] \right)
		,
		&&
		\epsilon_\alpha
	    = \Ocal \left( \alpha \left[\tilde{\eta} + \tilde{\eta}^2 K^2 \tau^2 \right] \left[ R(\{\Lcal^n\}) + 
	    \max q_i(n) \Sigma \right] \right)
	    ,
	\end{align}
	\begin{align}
	    \epsilon_\beta = \Ocal\left(\beta \left[\tilde{\eta} + \tilde{\eta}^2 K^2 \tau^2 \right] \left[ R(\{\Lcal^n\}) + \Sigma\right]\right)
	    ,
	    &&
    	\tilde{\eta} 
    	= \eta_g \eta_l
    	,
    	&&
    	\beta 
    	\coloneqq \max\{d_i(n) - \alpha q_i(n)\}
    	,
	\end{align}
    and $\Ocal$ accounts for numerical constants and the loss function Lipschitz smoothness $L$.

\end{theorem}
Theorem \ref{theo:convergence_convex} is proven in Appendix \ref{app:sec:convergence_convex}. The convergence guarantee provided in Theorem \ref{theo:convergence_convex} is composed of 5 terms: $R(\{\Lcal^n\})$, 
$\epsilon_F$, 
$\epsilon_K$,
$\epsilon_\alpha$,
$\epsilon_\beta$.
In the following, we describe these terms and explain their origin in a given optimization scheme.

\textbf{Optimized expected residual $R(\{\Lcal^n\})$}. 
The residual $R(\{\Lcal^n\})$ 
quantifies the sensitivity of $\Lcal^n$ between its optimum $\bar{\vtheta}^n$ and the optimum $\bar{\vtheta}$ of the overall expected minimized problem across optimization rounds $\tilde{\Lcal}$. 
As such, the residual accounts for the heterogeneity in the history of optimized problems, and is minimized to 0 when the same optimization problem is minimized at every round $n$, i.e. $\Lcal^n = \tilde{\Lcal}$. This condition is always satisfied when clients have identical data distributions, but requires for the server to set properly every client aggregation weight $d_i(n)$ in function of the server waiting time policy $\Delta t^n$ and the clients hardware capabilities $T_i^n$ in the general case (Section \ref{subsec:sufficient_conditions} and \ref{subsec:relaxed_sufficient}). 

\textbf{Initialization quality $\epsilon_{F}$}. 
$\epsilon_{F}$ only depends of the quality of the initial model $\vtheta^0$ through its distance with respect to the optimum $\bar{\vtheta}$ of the overall expected minimized problem across optimization rounds $\tilde{\Lcal}$. 
This convergence term can only be minimized by performing as many serial SGD steps $K N$.

\textbf{Clients data heterogeneity $\epsilon_K$}. 
This term accounts for the disparity in the clients updated models, and is proportional to the clients amount of local work $K$ (quadratically) and to the heterogeneity of their data distributions $\Zcal_i$ through $\Sigma_1$.  
When $K=1$, every client perform its SGD on the same model, which reduces the server aggregation to a traditional centralized SGD. We retrieve $\epsilon_K = 0$. 

\textbf{Gradient delay $\tau$ through $\epsilon_\alpha$ and $\epsilon_\beta$}.
Decreasing the server time policy $\Delta t^n$ allows faster optimization rounds but decreases a client's participation probability $\mathbb{P}( T_i^n \le \Delta t^n )$ resulting in an increased maximum answering time $\tau$. 
In turn, we note that $\epsilon_\alpha$ and $\epsilon_\beta$ are quadratically proportional to the maximum amount of serial SGD $K \tau$.
This latter terms quantifies the maximum amount of SGD integrated in the global model $\vtheta^n$.


\subsection{Sufficient Conditions for Minimizing the Federated Problem (\ref{eq:original_problem_general})}\label{subsec:sufficient_conditions}

Theorem \ref{theo:convergence_convex} provides convergence guarantees for the history of optimized models $\{\Lcal^n\}$.
Under the same assumptions of Theorem \ref{theo:convergence_convex}, we can provide convergence guarantees for the original FL problem $\Lcal(\vtheta)$ (proof in Appendix \ref{app:sec:theo_bias}).

\begin{theorem}\label{theo:bias_convex}
	Under the same conditions of Theorem \ref{theo:convergence_convex}, we have
	\begin{multline}
	\frac{1}{N}\sum_{ n = 0 }^{N-1}\frac{1}{K}\sum_{ k = 0 }^{K-1} \E{\norm{\nabla \Lcal (\vtheta^{n, k})}^2}
	\\
	\le 
	\Ocal \left( R(\{\Lcal^n\}) \right) 
	+ P(\{\Lcal^n\})
	+ U(\{\Lcal^n\})
	+ \Ocal \left( \epsilon_F \right)
	+ \epsilon_K
	+ \epsilon_\alpha
	+ \epsilon_\beta
	,
	\end{multline}	
	where 
	\begin{equation}
	P(\{\Lcal^n\})
	= \Ocal \left( \frac{1}{N}\sum_{n=0}^{N-1}\chi_n^2 \sum_{j \in W_n} \tilde{s}_j(n) \left[ F_j(\bar{\vtheta}^n) - F_j(\vtheta_j^*) \right] \right)
	,
	\end{equation}
	\begin{equation}
	U(\{\Lcal^n\})
	= \Ocal \left( \frac{1}{N}\sum_{n=0}^{N-1}  \frac{1}{K} \sum_{k=0}^{K-1} \sum_{j \notin W_n} r_j\left[ \E{F_j(\vtheta^{n , k})} - F_j(\vtheta_j^*) \right] \right)
	,
	\end{equation}
	$\chi_n^2 = \sum_{j\in W_n} (r_j - \tilde{s}_j(n))^2/\tilde{s}_j(n)$, 
	and $W_n = \{j : s_j(n) > 0\}$.
\end{theorem}
Theorem \ref{theo:bias_convex} provides convergence guarantees for the optimization problem (\refeq{eq:original_problem_general}). We retrieve the components of the convergence bound of Theorem \ref{theo:convergence_convex}.
The terms $\epsilon_F$ to $\epsilon_\tau$ can be mitigated by choosing an appropriate local learning rate $\eta_l$, but the same cannot be said for $R(\{\Lcal^n\})$, $P(\{\Lcal^n\})$, $U(\{\Lcal^n\})$. 
Behind these three quantities, Theorem \ref{theo:bias_convex} shows that proper expected representation of every dataset type is needed, i.e. $s_j (n) = r_j$. Indeed, if a client is poorly represented, i.e. $s_j (n) \neq r_j$, then $ R(\{\Lcal^n\})>0$ and $P(\{\Lcal^n\})>0$, while if a client is not represented at all, i.e. $s_j(n) =0$, then $U(\{\Lcal^n\})>0$. Therefore, we propose, with Corollary \ref{cor:bias_distributions}, sufficient conditions for any FL optimization scheme satisfying Algorithm \ref{alg:FL_with_delayed_gradients} to converge to the optimum of the federated problem (\refeq{eq:original_problem_general}).

\begin{corollary}\label{cor:bias_distributions} 
	Under the conditions of Theorem \ref{theo:convergence_convex}, if every client data distribution satisfies $\tilde{s}_j(n) = r_j$, the following convergence bound for optimization problem (\refeq{eq:original_problem_general}) can be obtained 
	\begin{align}
	\frac{1}{N}\sum_{ n = 0 }^{N-1}\frac{1}{K}\sum_{ k = 0 }^{K-1} \left[ \E{\Lcal(\vtheta^{n, k})} - \Lcal(\vtheta^*) \right]
	& \le 
	\epsilon_F 
	+ \epsilon_K
	+ \epsilon_\alpha
	+ \epsilon_\beta
	.
	\end{align}	
\end{corollary}

\begin{proof}
	Follows directly. $\tilde{s}_j(n) = r_j$ implies $\chi_n^2 = 0$, $W_n = \emptyset$, $\Lcal^n = q(n) \Lcal$, and $\bar{\vtheta}^n = \vtheta^*$.
\end{proof}

These theoretical results provide relevant insights for different FL scenarios. 

\textbf{iid data distributions, $\mathcal{Z}_i = \mathcal{Z}$}. Consistently with the extensive literature on synchronous and asynchronous distributed learning, when clients have data points sampled from the same data distribution, FL always converges to its optimum (Corollary \ref{cor:bias_distributions}). Indeed, $\tilde{s}_j(n) = r_j = 1$ regardless of which clients are participating, and what importance $p_i$ or aggregation weight $d_i(n)$ a client is given.

\textbf{non-iid data distributions}. The convergence of FL to the optimum requires to optimize by considering every data distribution type fairly at every optimization round, i.e. $\tilde{s}_j(n) = r_j$  (Corollary \ref{cor:bias_distributions}). This condition is weaker than requiring to treat fairly every client at every optimization round, i.e. $q_i(n) = p_i$.
Ideally, only one client per data type needs to have a non-zero participating probability, i.e. $\mathbb{P}( T_i^n \le \Delta t^n )>0$, and an appropriate $d_i(n)$ such that $\tilde{s}_j(n) = r_j$ is satisfied. In practice, knowing the clients data distribution is not possible. Therefore, ensuring FL convergence to its optimum requires at every optimization round $\tilde{q}_i(n) = p_i$ \citep{FedNova}.

We provide in Example \ref{exa:bias} an illustration on these results based on quadratic loss functions to show that considering fairly data distributions is sufficient for an optimization scheme satisfying Algorithm \ref{alg:FL_with_delayed_gradients} to converge to the optimum of the optimization problem (\refeq{eq:original_problem_general}), since $\tilde{s}_j(n) = r_j$ is satisfied at every optimization round, while $\tilde{q}_i(n) \neq p_i$ may not be satisfied. 

\begin{example}\label{exa:bias}
	Let us consider four clients with data distributions such that their loss can be expressed as $\Lcal_i(\vtheta) = \frac{1}{2}\norm{\vtheta - \vtheta_i^*}^2$ with $\vtheta_1^* = \vtheta_2^*$ ($\Zcal_1$), $\vtheta_3^* = \vtheta_4^*$ ($\Zcal_2$), and identical client importance, i.e. $p_i = 1/4$. 
	Therefore, each data type has identical importance, i.e. $r_j = 1/2$, and the optimum satisfies $\vtheta^* = \frac{1}{2}[\vtheta_1^* + \vtheta_3^*]$.
	We consider that clients with odd index participate at odd optimization rounds while the ones with even index at even optimization rounds, i.e. $q_1^{2n+1} = q_3^{2n+1} = q_2^{2n} = q_4^{2n} = 1/2$ and $q_1^{2n} = q_2^{2n} = q_3^{2n+1} = q_4^{2n+1} =  0$ which gives $\tilde{s}_1(n) = \tilde{s}_2(n) = 1 /2$ and $\tilde{q}_i(n)=0$ or $\tilde{q}_i(n)=1/2$ but not $\tilde{q}_i(n)=1/4$.
	With $\eta_g = 1$, equation (\refeq{eq:aggreg_SCAFFOLD_general}) can be rewritten as
	\begin{equation}
		\vtheta^{n+2} = \vtheta^{n+1} + \frac{1}{2}\left[ ( \vtheta_1^{n+1} - \vtheta^n) + ( \vtheta_3^{n+1} - \vtheta^n)\right].
	\label{exa:eq:A1}
	\end{equation}
	Clients update can be rewritten as 
	$
	\vtheta_i^{n+1} - \vtheta^n
	= \phi (\vtheta_i^* - \vtheta^n),
	$
	where $\phi = 1 - ( 1- \eta_l)^K$. Equation (\refeq{exa:eq:A1}) can thus be rewritten as
	\begin{equation}
	\vtheta^{n+2} 
	- \vtheta^{n+1}  
	+ \phi \vtheta^n 
	= \phi \vtheta^*.
	\label{exa:eq:A2}
	\end{equation}
	Solving equation (\refeq{exa:eq:A2}) proves FL asymptotic convergence to the optimum $\vtheta^*$.
	
\end{example}

\subsection{Relaxed Sufficient Conditions for Minimizing the Federated Problem (\ref{eq:original_problem_general})}\label{subsec:relaxed_sufficient}

Theorem \ref{theo:bias_convex} holds for any client's update time $T_i$ and optimization scheme satisfying Algorithm \ref{alg:FL_with_delayed_gradients}, and provides finite convergence guarantees for the optimization  problem (\refeq{eq:original_problem_general}). Corollary \ref{cor:bias_distributions} shows that for the asymptotic convergence of FL, data distribution types should be treated fairly in expectation, i.e. $\tilde{s}_j(n) = r_j$.
This sufficient condition is not necessarily realistic, since the server cannot know the clients data distributions and participation time, and thus needs to give to every client an aggregation weight $d_i(n)$ such that $\tilde{q}_i(n) = p_i$ without knowing $T_i$.

In Example \ref{exa:bias}, we note that we have $\frac{1}{2}\left[q_i^{2n} + q_i^{2n+1}\right] = p_i$. Therefore, every client is given proper consideration every two optimization rounds.
Based on Example \ref{exa:bias}, in Theorem \ref{theo:relaxed_sufficient_conditions} we provide weaker sufficient conditions than the ones of Corollary \ref{cor:bias_distributions}. To this end, we assume that clients are considered with identical importance across $W$ optimization rounds (Assumption \ref{ass:window}) and that clients gradients are bounded (Assumption \ref{ass:bounded}).

\begin{assumption}[Window]\label{ass:window}	
	$\exists W \ge 1$ such that $\forall s,\ \frac{1}{W}\sum_{n=sW}^{(s+1)W-1} q_i(n) = q_i$.
\end{assumption}
With Assumption \ref{ass:window}, we assume that over a cycle of $W$ aggregations, the sum of the clients expected aggregation weights $q_i(n)$ are constant. 
By definition of $q_i$, Assumption \ref{ass:window} is always satisfied with $W = N$. Also, by construction, we have $ W \ge \tau$.
We note  that Assumption \ref{ass:window} is made on a series of windows of size $W$ and not for any window of size $W$.

\begin{assumption}[Bounded Gradients]\label{ass:bounded} $\exists B>0$ such that	
	$ \E{\norm{\nabla \Lcal_i(\vx)}} \le B$
	.
\end{assumption}
Gradient clipping is a typical operation performed during the optimization of deep learning models to prevent exploding gradients. A pre-determined gradient threshold $B$ is introduced, and gradients with norm exceeding this threshold are clipped to the norm $B$. Therefore, using Assumption \ref{ass:bounded} and the subadditivity of the norm, the distance between two consecutive global models can be bounded by
\begin{equation}
\E{\norm{\vtheta^{n+1} - \vtheta^{n}}}
\le \eta_g \sum_{ i = 1 }^M q_i(n) \E{\norm{\vtheta_i^{\rho_i(n)+1} - \vtheta^{\rho_i(n)}} }
\le \eta_g \eta_l q(n) K B
,
\end{equation}
which, thanks to the convexity of the clients loss function and to the Cauchy Schwartz inequality, gives 
\begin{equation}
\E{\Lcal_i(\vtheta^{n+1})} - \E{\Lcal_i(\vtheta^{n})}
\le \E{\inner{\nabla  \Lcal_i(\vtheta^{n+1})}{\vtheta^{n+1} - \vtheta^n}}
\le \eta_g \eta_l q(n) B^2 K
\label{eq:ass_bounded}
.
\end{equation}
Finally, using equation (\refeq{eq:ass_bounded}) and Assumption \ref{ass:window}, the performance history on the original optimized problem can be bounded as follows
\begin{multline}
\sum_{n= s W}^{(s+1)W - 1}\sum_{ k = 0 }^{K-1} q_i \E{\Lcal_i(\vtheta^{(n, k)})} 
\le 
\sum_{n= s W}^{(s+1)W - 1}\sum_{ k = 0 }^{K-1} q_i(n) \left[\E{\bar{\Lcal}(\vtheta^{(n, k)})} 
+  \eta_g \eta_l K (W-1) B^2 \right]
.
\label{eq:window_bound}
\end{multline}


\begin{theorem}\label{theo:relaxed_sufficient_conditions}
	Under the conditions of Theorem \ref{theo:convergence_convex}, Assumptions \ref{ass:window} and \ref{ass:bounded}, and considering that $W$ is a divider of  $N$, we get the following convergence bound for the optimization problem (\refeq{eq:surrogate_problem}):
	\begin{equation}
	\frac{1}{N}\sum_{ n = 0 }^{N-1}\frac{1}{K}\sum_{ k = 0 }^{K-1} \left[ \E{\bar{\Lcal}(\vtheta^{n, k})} - \bar{\Lcal}(\bar{\vtheta}) \right]
	\le \epsilon
	\coloneqq 
    \epsilon_F 
	+ \epsilon_K
	+ \epsilon_\alpha
	+ \epsilon_\beta
	+ \epsilon_W
		,
	\end{equation}
	where $\epsilon_W = \Ocal( \eta_g \eta_l (W-1) K)$. 
	Furthermore, we obtain the following convergence guarantees for the federated problem (\refeq{eq:original_problem_general}):
	\begin{align}
	\frac{1}{N}\sum_{ n = 0 }^{N-1}\frac{1}{K}\sum_{ k = 0 }^{K-1} \E{\norm{\nabla \Lcal (\vtheta^{n, k})}^2}
	& \le 
	\epsilon
	+ \Ocal (\chi^2 [\bar{\Lcal}(\bar{\vtheta}) -\sum_{j =1}^J s_j F_j(\vtheta_j^*)])
	,
	\end{align}	
	where $\chi^2 = \sum_{ j = 1 }^J \frac{\left(r_j - \tilde{s}_j\right)^2}{\tilde{s}_j}$.	
\end{theorem}

\begin{proof}
	\begin{align}
	&\frac{1}{N}\sum_{ n = 0 }^{N-1}\frac{1}{K}\sum_{ k = 0 }^{K-1} \left[ \E{\bar{\Lcal}(\vtheta^{n, k})} - \bar{\Lcal}(\bar{\vtheta}) \right]
	\nonumber\\
	& \le \frac{1}{N}\sum_{ n = 0 }^{N-1}\frac{1}{K}\sum_{ k = 0 }^{K-1}  q_i(n)\left[\E{\Lcal_i(\vtheta^{n, k})} + \tilde{\eta}  K (W - 1) B^2  \right] 
	- \bar{\Lcal}({\bar{\vtheta}})
	\\
	& \le R(\{\Lcal^n\})
	+ \epsilon 
	+ \frac{1}{N} \sum_{n=0}^{N-1} \Lcal^n(\bar{\vtheta}^n) 
	- \bar{\Lcal}(\bar{\vtheta}) 
	= \epsilon 
	,
	\end{align}
	where we use equation (\refeq{eq:window_bound}) for the first inequality and Theorem \ref{theo:convergence_convex} for the second inequality.
	
	Finally, we can obtain convergence guarantees on the optimization problem (\refeq{eq:original_problem_general}) with Theorem \ref{theo:bias_convex} by considering the minimization of the optimization problem $\bar{\Lcal}$. Therefore, the bound of Theorem \ref{theo:bias_convex} can be simplified noting that $\Lcal^n = \bar{\Lcal}$, $\bar{\vtheta}^n = \bar{\vtheta}$, $W_n = \emptyset$, $\chi_n^2 = \chi^2$, and by adding $\epsilon_W$, which completes the proof.
\end{proof}

Theorem \ref{theo:relaxed_sufficient_conditions} shows that the condition $\tilde{s}_j = r_j$ is sufficient to minimize the optimization problem (\refeq{eq:original_problem_general}). In practice, for privacy concerns, clients may not want to share their data distribution with the server, and thus the relaxed sufficient condition becomes $\tilde{q}_i = p_i$. This condition is weaker than the one obtained with Corollary \ref{cor:bias_distributions}, at the detriment of a looser convergence bound including an additional asymptotic term $\epsilon_W$ linearly proportional to the window size $W$. Therefore, for a given learning application, the maximum local work delay $\tau$ and the window size $W$ need to be considered when selecting an FL optimization scheme satisfying Algorithm \ref{alg:FL_with_delayed_gradients}.
Also, the server needs to properly allocate clients aggregation weight $d_i(n)$ such that Assumption \ref{ass:window} is satisfied while keeping at a minimum the window size $W$. We note that $W$ depends of the considered FL optimization scheme and clients hardware capabilities.
Based on the results of Theorem \ref{theo:relaxed_sufficient_conditions}, in the following section, we introduce \textsc{FedFix}, a novel asynchronous FL setting based on a waiting policy over fixed time windows $\Delta t^n$. 

Finally, the following example illustrates a practical application of the condition $\tilde{q}_i = p_i$. 
\begin{example}\label{ex:relaxed_sufficient}
	We consider two clients, $i=1, 2$, with $\Lcal_i(\vtheta) = \frac{1}{2}\norm{\vtheta - \vtheta_i^*}^2$ where clients have identical importance, i.e. $p_1 = p_2 = 1/2$. Client 1 contributes at even optimization rounds and Client 2 at odd ones, i.e. $q_1^{2n} =  q_1$, $q_2^{2n+1} = q_2$,  and $q_1^{2n+1} = q_2^{2n} = 0$. 	
	Hence, we have
	\begin{align}
	\vtheta^n 
	&\xrightarrow{n \to \infty } \frac{q_1 \vtheta_1^* + q_2 \vtheta_2^*}{q_1 + q_2}
	,
	\end{align}
	which converges to the optimum of problem (\refeq{eq:original_problem_general}) if and only if $ \frac{1}{2}\left[ \tilde{q}_i^{2n} + \tilde{q}_i^{2n+1} \right] = p_i $ (Theorem \ref{theo:relaxed_sufficient_conditions}).
	
\end{example}

The conditions of Corollary \ref{cor:bias_distributions} and Theorem \ref{theo:relaxed_sufficient_conditions} are equivalent when $W=1$, where we retrieve $\epsilon_W=0$. They are also equivalent when clients have the same data distributions, and we retrieve $\tilde{s}_j = r_j = 1$ at every optimization round, which also implies that $W=1$.

The convergence guarantee proposed in Theorem \ref{theo:relaxed_sufficient_conditions} depends on the window size $W$, and to the maximum amount of optimizations needed for a client to update its work $\tau$. 
We provide sufficient conditions in Corollary \ref{cor:sufficient_learning_rate} for the parameters $W$, and $\tau$, such that an optimization scheme satisfying Algorithm \ref{alg:FL_with_delayed_gradients} converges to the optimum of the optimization problem (\refeq{eq:original_problem_general}).

\begin{corollary}\label{cor:sufficient_learning_rate}
	Let us assume there exists $a \ge 0$ and $b \ge 0$ such that $W = \Ocal(N^a)$, $\tau = \Ocal(N^b)$, and $\eta_l \propto N^{-c}$. The convergence bound of Theorem \ref{theo:relaxed_sufficient_conditions} asymptotically converges to 0 if
	\begin{equation}
		W = o(N), 
		\tau = o(N), 
		\text{ and }
		\max (a, b) < c < 1
	\end{equation}
\end{corollary}

\begin{proof}
	The bound of Theorem \ref{theo:relaxed_sufficient_conditions} converges to 0 if the following quantities also do: $\eta_l W$, $\frac{1}{\eta_l N}$, $\tau \eta_l$, $\eta_l$. We get the following conditions on $a$, $b$, and $c$: $ - c + a< 0$, $ c- 1< 0$, $ b - c< 0$, $ - c < 0$, which completes the proof.
\end{proof}

By construction and definition of $q_i$, Assumption \ref{ass:window} is always satisfied with $W=N$. However, Corollary \ref{cor:sufficient_learning_rate} shows that when $W=N$, no learning rate $\eta_l$ can be chosen such that the learning process converges to $\vtheta^*$. 
Also, Corollary \ref{cor:sufficient_learning_rate} shows that Assumption \ref{ass:answering_time} can be relaxed. Indeed, Assumption \ref{ass:answering_time} implies that $\tau = \Ocal (1)$ and Corollary \ref{cor:sufficient_learning_rate} shows that $\tau = o(N)$ is sufficient.
We show in Section \ref{sec:applications} that all the considered optimization schemes satisfy $\tau = \Ocal (1)$ and $W = \Ocal (1)$, and also depend of the clients hardware capabilities and amount of participating clients $M$.
\section{Applications} \label{sec:applications}

In this section, we show that the formalism of Section \ref{sec:background} can be applied to a wide-range of optimization schemes, demonstrating the validity of the conclusions of Corollary \ref{cor:bias_distributions} and Theorem \ref{theo:relaxed_sufficient_conditions} (Section \ref{sec:convergence}). 
When clients have identical data distributions, the sufficient conditions of Corollary \ref{cor:bias_distributions} are always satisfied (Section \ref{sec:convergence}). 
In the heterogeneous case, these conditions can also (theoretically) be satisfied. It suffices that every client has a non-null participation probability, i.e. $\mathbb{P}( T_i^n \le \Delta t^n) >0$ such that there exists an appropriate $d_i(n)$ satisfying $\tilde{q}_i(n) = p_i$.
Yet, in practice clients generally may not even know their update time distribution $\mathbb{P}( T_i^n)$ making the computation of $d_i(n)$ intractable. In what follows, we thus focus on Theorem \ref{theo:relaxed_sufficient_conditions} to obtain the close-form of $\epsilon$, which only requires from the server to know the clients time $\tau_i$.

Theorem \ref{theo:relaxed_sufficient_conditions} provides a close-form for the convergence bound $\epsilon$ of an optimization scheme in function of the amount of server aggregation rounds $N$.
We first introduce in Section \ref{subsec:specifications} our considerations for the clients hardware and data to instead express $\epsilon$ in function of the training time $T$.
The quantity $\epsilon$ also depends on the optimization scheme time policy $\Delta t^n$ through $\alpha$, $\beta$ and $\tau$, and on the clients data heterogeneity through $R(\{\Lcal^n\})$ and $W$. We provide their close-form for synchronous \textsc{FedAvg} (Section \ref{subsec:synchronous}), asynchronous \textsc{FedAvg} (Section \ref{subsec:async_FL}), 
and \textsc{FedFix} (Section \ref{subsec:FedFix}), a novel asynchronous optimization scheme motivated by Section \ref{subsec:relaxed_sufficient}.
Finally, in Section \ref{subsec:generalization}, we show that the conclusions drawn for synchronous/asynchronous \textsc{FedAvg} and \textsc{FedFix} can also be extended to other distributed optimization schemes with delayed gradients.
Of course, similar bounds can seamlessly be derived for centralized learning and client sampling, which we differ to Appendix \ref{app:sec:applying} to focus on asynchronous FL in this section.

\subsection{Heterogeneity of clients hardware and data distributions}\label{subsec:specifications}

\textbf{Clients importance.}
We restrict our investigation to the case where clients have identical aggregation weights during the learning process, i.e. $d_i(n) = d_i$.  
We also consider identical client importance $p_i = 1/M$.
We can therefore define the averaged optimum residual $\Sigma$ defined as the average of the clients SGD evaluated on the global optimum, i.e.
\begin{equation}
	\Sigma
	\coloneqq
	\frac{1}{M} \sum_{ i = 1 }^M \EE{\vxi_i}{\norm{\nabla \Lcal_i(\vtheta^*, \vxi_i)}^2}
	.
\end{equation}
When clients have identical data distributions, $\Sigma$ can be simplified as $\Sigma = \EE{\vxi}{\norm{\nabla \Lcal(\vtheta^*, \vxi)}^2}$, and $\Sigma =0$ when clients perform GD. We note that in the DL and FL literature $\Sigma$ is often simplified by assuming bounded variance of the stochastic gradients, i.e. $\Sigma \le \sigma^2$, where $\sigma^2$ is the bounded variance of any client SG.


\textbf{Clients computation time.} 
In the rest of this work, we consider that clients guarantee reliable computation and communication, although with heterogeneous hardware capabilities, i.e. $\exists \tau_i \in \mathbb{R},\text{ s.t. }T_i = \tau_i$.
Without loss of generality, we assume that clients are ordered by increasing $\tau_i$, i.e. $\tau_i \le \tau_{i+1}$, where the unit of $\tau_i$ is such that $\tau_i$ is an integer. 
In what follows, we provide the close form of $d_i$ for all the considered optimization schemes. This derivation still holds for applications where clients have unreliable hardware capabilities that can be modeled as an exponential distribution, i.e. $T_i \sim \exp(\tau_i^{-1})$ which gives $\E{T_i} = \tau_i$.

\textbf{Clients data distributions.}
Unless stated otherwise, we will consider the FL setting where each client has its unique data distribution. 
Therefore, clients have heterogeneous hardware and non-iid data distributions. 
The obtained results can be simplified for the DL setting where a dataset is made available to $M$ processors.
In this special case, clients have iid data distributions ($\Zcal_i = \Zcal_1$)
, and identical computation times ($\tau_i = \tau_1$, $W = M$, and $\tau = M-1$).

\textbf{Learning rates.} 
For sake of clarity, we ignore the server learning rate when expressing the convergence bounds $\epsilon$, i.e. $\eta_g = 1$. 
Also, we consider a local learning rate $\eta_l$ inversely proportional to the serial amount of SGD included in the global model, i.e. $\eta_l \propto 1 / \sqrt{KN}$, consistently with the rest of the distributed optimization literature.

We propose Table \ref{tab:summary_comparison} to summarize the close form or bounds of the different parameters used in Section \ref{sec:convergence}.

\begin{table}
	\begin{center}
		\begin{tabular}{ | c || c | c | c| }
			\hline
			& Sync. \textsc{FedAvg} 
			& Async. \textsc{FedAvg} 
			& \textsc{FedFix}\\
			\hline
			$d_i$   
			& $= p_i$ 
			& $= \left[\sum_{ i = 1 }^M \frac{1}{\tau_i}\right] \tau_i p_i$
			& $= \ceil{ \tau_i /\Delta t } p_i$
			\\
			$N$ 
			&   $T/\tau_M $
			& $\sum_{i=1}^M T/\tau_i $
            & $ T/\Delta t $
			\\
			$\Delta t$ 
			& = $\max T_i^n$ 
			& = $\min T_i^n$
			& = $\Delta t$
			\\
			$\alpha$ 
			& $1$ 
			& 0
			& 1
			\\
			$\beta$
			& 0 
			& $\max d_i \le \tau_m / \tau_0$
			& 0
			\\
			$\tau$ 
			& 0 
			& $\Omega(M)$, $\Ocal(M \tau_M / \tau_0)$
			& 0, $\floor{\tau_m / \tau_0}$
			\\
			$W$ 
			& 1 
			& $\Omega(M)$, $\Ocal(M (\tau_M)^M)$
			& 1, $M \ceil{\tau_m / \tau_0}^M$
			\\
			$R(\{\Lcal^n\})$ 
			& $= 0$
			& $= \frac{1}{M} \sum_{ i = 1 }^M \left[\Lcal_i(\vtheta^*) - \Lcal_i(\vtheta_i^*)\right]$
			& $\le \frac{1}{M} \sum_{ i = 1 }^M \left[\Lcal_i(\vtheta^*) - \Lcal_i(\vtheta_i^*)\right]$
			\\
			\hline
		\end{tabular}
	\end{center}

\caption{The different variables used to account for the importance of clients or data distributions at every optimization round and during the full FL process. For $\tau$ and $W$, we give two values which correspond to their respective lower and upper bound.}
\label{tab:summary_comparison}
\end{table}

\subsection{\textsc{FedAvg}, Synchronous Federated Learning}\label{subsec:synchronous}

As described for \textsc{FedAvg} in Section \ref{subsec:agg_schem}, at every optimization round, the server sends to the clients the current global model to perform $K$ SGD steps on their own data before returning the resulting model to the server. Once every client performs its local work, the new global model is created as the weighted average of the clients contribution. 
The time required for an optimization step is therefore the one of the slowest client ($\Delta t^n = \max_i(T_i^n)$), and every client is considered ($\mathbb{P}( T_i^n \le \Delta t^n) = 1$). Hence, $\alpha= 1$, $\beta = 0$, and setting $d_i = p_i$ is sufficient to satisfy the conditions of Corollary \ref{cor:bias_distributions} (and thus the ones of Theorem \ref{theo:relaxed_sufficient_conditions}) ensuring that FL converges to its optimum \citep{FedNova}. The term $\epsilon$ then reduces to
\begin{align}
\epsilon_{\textsc{FedAvg}}
&=\frac{1}{\sqrt{K N} } \norm{\vtheta^0 -\vtheta^*}^2 
+ \Ocal \left(\frac{K-1}{N} \Sigma \right)
+ \Ocal\left(\frac{1}{\sqrt{K N} } \frac{1}{M} \Sigma 
\right)
\label{eq:epsilon_sync_FL}
.
\end{align}
The second element of equation (\refeq{eq:epsilon_sync_FL}) accounts for the clients update disparity through their amount of local work $K$ between two server aggregations, and is proportional to the SG variance $\Sigma$.
The third element benefits of the distributed computation by being proportional to $1/M$. Equation (\refeq{eq:epsilon_sync_FL}) is consistent with literature on convex distributed optimization with \textsc{FedAvg} including \cite{FedNova, Khaled2020}.

\subsection{Asynchronous \textsc{FedAvg}}
\label{subsec:async_FL}

With \textsc{FedAvg}, every client waits for the slowest one to perform its local work, and cannot contribute to the learning process during this waiting time. To remove this bottleneck, with asynchronous \textsc{FedAvg}, the server creates a new global model whenever it receives a client contribution before sending it back to this client. 
For in depth discussion of Asynchronous \textsc{FedAvg}, please refer to \cite{AFLsurvey}. 

With asynchronous \textsc{FedAvg},
clients always compute their local work but each on a different global model, giving $\Delta t^n = \min_i T_i^n$, $\alpha = 0$, and $\beta = \max_i d_i$. In addition, while the slowest client updates its local work, other clients performs a fix amount of updates (up to $\ceil{\tau_M/\tau_i}$). By scaling this amount of updates by the amount of clients sending updates to the server, we have 
\begin{equation}
    \tau = \Ocal\left(\frac{\tau_M}{\tau_0}(M-1)\right)
    .
\end{equation} 

We define $lcm(\{x_i\})$ the function returning the least common multiplier of the set of integers $\{x_i\}$. Hence, after every $\nu \coloneqq lcm(\{\tau_i\})$ time, each client has performed $\nu /\tau_i$ optimization rounds and the cycle of clients update repeats itself. Thus, the smallest window $W$ satisfies
\begin{equation}
	W
	= \sum_{ i = 1 }^M \nu / \tau_i
	.
\end{equation}
By construction, $\nu \ge \tau_M$ and thus $W = \Omega(M)$, with $W=M$ when clients have homogeneous hardware ($\tau_M = \tau_0$). In the worse case, every $\tau_i$ is a prime number, and we have $\nu / \tau_i \le \left(\tau_M\right)^{M-1}$, which gives $W = \Ocal(M \left( \tau_M\right)^{M-1})$. 
In a cycle of $W$ optimization rounds, every client participates $\nu /\tau_i$ times to the creation of a new global model. 
Therefore, we have $q_i(n) = d_i$ for the $\nu / \tau_i$ participation of client $i$, and $q_i(n) = 0$ otherwise.
Hence, the sufficient conditions of Theorem \ref{theo:relaxed_sufficient_conditions} are satisfied when 
\begin{equation}
\label{eq:aggreg_weights_async}
q_i
= \frac{1}{W}\sum_{ n = kW }^{(k+1)W - 1 }q_i(n) 
= \frac{1}{\sum_{ i = 1 }^M \nu / \tau_i}\frac{ \nu}{\tau_i }
d_i
= p_i
\Rightarrow
d_i
= \left[\sum_{ i = 1 }^M \frac{1}{\tau_i}\right] \tau_i p_i
.
\end{equation}
The client weight calculated in equation (\refeq{eq:aggreg_weights_async}) is constant and only depends on the client importance $p_i$ (set and thus known by the server), and on the clients computation time $\tau_i$ (eventually estimated by the server after some clients updates). The condition on $d_i$ can be further simplified by accounting for the server learning rate $\eta_g$. Coupling equation (\refeq{eq:aggreg_SCAFFOLD_general}) with equation (\refeq{eq:aggreg_weights_async}) gives $\eta_g d_i \propto \tau_i p_i$, which is sufficient to guarantee the convergence of asynchronous FL to its optimum. 
Finally, by bounding $\tau_i$, we also have $\beta = \max_i d_i \le \tau_M/ \tau_0$, bounded the hardware computation time heterogeneity.

The disparity between the optimized objectives $R(\{\Lcal^n\})$ at different optimization rounds also slows down the learning process. Indeed, at every optimization round, only a single client can participate with probability $1$. As such, we have $\Lcal^n = d_i \Lcal_i$ which, thanks to the assumption $p_i = 1/M$, yields 
\begin{equation}
	R( \{\Lcal^n\})
	= \frac{1}{M} \sum_{ i = 1 }^M \left[\Lcal_i(\vtheta^*) - \Lcal_i(\vtheta_i^*)\right]
	.
\end{equation}
Finally, we simplify the close-form of $\epsilon$ (Theorem \ref{theo:relaxed_sufficient_conditions}) for asynchronous \textsc{FedAvg} to get
\begin{align}
\epsilon_{Async}
&= 
\frac{1}{\sqrt{KN}} \norm{\vtheta^0 -\vtheta^*}^2
+ \Ocal \left(\frac{K-1}{N} \Sigma \right)
+ \Ocal\left(\frac{\tau_M}{\tau_0} \frac{1}{\sqrt{KN}} \left[ R(\{\Lcal^n\}) + \Sigma \right]\right)
\nonumber\\
&
+ \Ocal\left( \left(\frac{\tau_M}{\tau_0}\right)^3 \frac{K}{N} M^2 \left[ R(\{\Lcal^n\}) + \Sigma \right]\right)
+ \Ocal\left(  \frac{1}{\sqrt{KN}} (W-1)\right)
\label{eq:epsilon_async}
.
\end{align}
With equation (\refeq{eq:epsilon_async}), we can compare synchronous and asynchronous \textsc{FedAvg}.
The first and second asymptotic terms are identical for the two learning algorithms, while the third asymptotic term is scaled by the hardware characteristics $\tau_M/\tau_0$ instead of $1/M$ in $\textsc{FedAvg}$, with the addition of a non null residual $R(\{\Lcal^n\})$ for asynchronous \textsc{FedAvg}. 
However, the fourth and fifth term are unique to asynchronous \textsc{FedAvg}, and explains why its convergence gets more challenging as the amount of clients $M$ increases.
The impact of the hardware heterogeneity is also identified through the importance of $\tau_M/ \tau_0$ in the third term. With no surprise, for a given optimization round, synchronous \textsc{FedAvg} outperforms its asynchronous counterpart. However, in $T$ time, the server performs 
\begin{equation}
    N = \sum_{i=1}^M T / \tau_i
\end{equation}
aggregations with asynchronous \textsc{FedAvg} against $T/ \tau_M$ for synchronous \textsc{FedAvg}. With asynchronous \textsc{FedAvg}, the server thus performs at least $M$ times more aggregations than with synchronous \textsc{FedAvg}. As a result, the first two terms of equation (\ref{eq:epsilon_async}), which are proportional to how good the initial model is $\norm{\vtheta_0 - \vtheta^*}$, decrease faster with asynchronous \textsc{FedAvg} at the detriment of an higher convergence residual coming for the two last terms.

\textbf{Comparison with asynchronous DL and \textsc{FedAvg} literature.} The convergence rates obtained in the convex distributed optimization literature relies on additional assumptions to ours, with which we retrieve their proposed convergence rate. 
To the best of our knowledge, only \cite{SlowLearnersAreFast} considers non-iid data distributions for the clients. When assuming $W = \Ocal(\tau)$ and $\eta_l \propto 1/ \sqrt{\tau N}$, we retrieve a convergence rate$\sqrt{\tau/ N}$. 

We also match convergence rates for literature with iid client data distributions and $K=1$. With $M = \Ocal(\sqrt{N})$, then 
we have $\Ocal(1/\sqrt{N})$ \citep{DistributedDelayed, AsyncParallelSGDforNonconvex}.
When $\eta_l = \Ocal( 1 /\tau \sqrt{KN})$, we retrieve $\tau/N + 1/\sqrt{N}$ \citep{ErrorFeedbackFramework, StichCriticalParameters}.

\subsection{\textsc{FedFix}}
\label{subsec:FedFix}


The analysis of asynchronous \textsc{FedAvg} (Section \ref{subsec:async_FL}) and its comparison with synchronous \textsc{FedAvg} (Section \ref{subsec:synchronous}), shows that asynchronous \textsc{FedAvg} is not scalable to large cohort of clients. We thus propose \textsc{FedFix} combining the strong points of synchronous and asynchronous \textsc{FedAvg}, where the server creates the new global model at a fixed time $t^n$ with the contributions received since $t^{n-1}$. Therefore, the server does not wait for every client, contrarily to synchronous \textsc{FedAvg}, and considers more than one client per aggregation to have more stable aggregations, contrarily to asynchronous \textsc{FedAvg}. We provide in Figure \ref{fig:illustration} an illustration of \textsc{FedFix} with two clients.

With \textsc{FedFix}, an iteration time $\Delta t^n = t^{n+1} - t^n$ is decided by the server and is independent from the clients remaining update time $T_i^n$. 
For sake of convenience, we further assume that the time between optimization rounds is identical, i.e. $\Delta t^n = \Delta t$, but the following results can be derived for other fixed time policies $\{\Delta t^n\}$. 
Therefore, $T_i^n$ and $T_j^n$ are independent, and so are $\omega_i$ and $\omega_j$, which gives $\alpha = 1$ and $\beta = 0$.

Every client sends an update to the server in $N_i' = \ceil{T_i / \Delta t }$ optimization steps. 
Contrarily to asynchronous \textsc{FedAvg}, we thus have $\tau = \ceil{\tau_m /\Delta t} = \Ocal(1)$, which is independent from the amount of participating clients $M$. 
In this case,  the smallest window $W$ satisfies $W = lcm(\{N_i'\})$, and clients update $W/ N_i'$ times their work to the server during the window $W$. Therefore, satisfying the conditions of Theorem \ref{theo:relaxed_sufficient_conditions} requires
\begin{equation}
d_i
= \ceil{ \frac{\tau_i }{\Delta t} } p_i
.
\label{eq:aggreg_weights_hybrid}
\end{equation}
With equation (\refeq{eq:aggreg_weights_hybrid}), we can notice the relationship between \textsc{FedFix} and synchronous or asynchronous \textsc{FedAvg}. When $\Delta t \ge \tau_i$, client $i$ participates to every optimization round and is thus considered synchronously, which gives $d_i = p_i$. When $\Delta_t \ge \tau_M$, we retrieve synchronous FL and $d_i = p_i$ for every client. On the contrary, for asynchronous FL, when $\Delta t \ll \tau_i$, we obtain $\ceil{\tau_i / \Delta t}\approx \tau_i /\Delta t$ and we retrieve $\eta_g d_i = \eta_g \left[ \tau_i / \Delta t \right] p_i \propto \tau_i p_i$.

%

Regarding the disparity between the local objectives $R\{\Lcal^n\}$, we know that a client participates to an optimization round with $q_i(n) =d_i$. We thus have $\Lcal^n = \sum_{ i \in S_n } d_i \Lcal_i$, where $S_n$ is the set of the participating clients at optimization step $n$. Considering that $\Lcal^n (\bar{\vtheta}^n)\ge \sum_{ i \in S_n } d_i \Lcal_i(\vtheta_i^*)$, the close form of \textsc{FedFix} is bounded by the one of of asynchronous \textsc{FedAvg}, i.e.
\begin{equation}
	R(\{\Lcal^n\})
	\le  \frac{1}{M} \sum_{ i = 1 }^M \left[\Lcal_i(\vtheta^*) - \Lcal_i(\vtheta_i^*)\right]
	.
\end{equation}
Finally, we simplify the close-form of $\epsilon$ (Theorem \ref{theo:relaxed_sufficient_conditions}) for \textsc{FedFix} to get
\begin{align}
	\epsilon_{\textsc{FedFix}}
	&= \frac{1}{ \sqrt{KN}}\E{ \norm{\vtheta^{0} - \vx}^2 }
	+ \Ocal \left(\frac{K-1}{N} \left[ R(\{\Lcal^n\}) + \Sigma \right] \right) 
	\nonumber\\
	& 
	+ \Ocal \left( \left[ \frac{1}{ \sqrt{KN}} +  \frac{K}{ N} \ceil{\frac{\tau_m}{ \Delta t}}^2 \right] \left[ R(\{\Lcal^n\}) + \ceil{\frac{\tau_m}{ \Delta t}} \frac{1}{M} \Sigma \right] \right)
	+ \Ocal\left(  \frac{1}{\sqrt{KN}} (W-1)\right) 
	.\label{eq:epsilon_FedFix}
\end{align}
The first two elements of equation (\refeq{eq:epsilon_FedFix}) are identical for \textsc{FedFix}, synchronous and asynchronous \textsc{FedAvg}. However, thanks to lower values for the different variables (cf Table \ref{tab:summary_comparison}), the last two asymptotic terms of the convergence bound are smaller for \textsc{FedFix} than for asynchronous \textsc{FedAvg}, equation (\refeq{eq:epsilon_FedFix}). Similarly, these two terms are larger with \textsc{FedFix} than with synchronous \textsc{FedAvg}.
The hardware heterogeneity and the amount of participating clients still impacts the convergence bound through $\ceil{\tau_M/ \Delta t}$ and $W$, but can be mitigated with proper selection of $\Delta t$. 
Therefore, after $N$ optimization rounds, synchronous \textsc{FedAvg} outperforms \textsc{FedFix} which outperforms in turn  asynchronous \textsc{FedAvg}. However, in $T$ time, the server performs 
$ N = T / \Delta t$
aggregations with \textsc{FedFix} against $T/ \tau_M$ for synchronous \textsc{FedAvg}. With asynchronous \textsc{FedAvg}, the server thus performs at least $\tau_M/\Delta t$ times more aggregations than with synchronous \textsc{FedAvg}. Overall, $\Delta t$ can be considered as the level of asynchronicity given to Algorithm \ref{alg:FL_with_delayed_gradients}, with \textsc{FedAvg} when $\Delta t  = \tau_M$ and asynchronous \textsc{FedAvg} when $\Delta t \ge \tau_M$.

In the DL case, clients have identical computation time ($\tau_1 = \tau_m$), and we retrieve the convergence bound of synchronous \textsc{FedAvg}.

In addition, we can increase the waiting time for the clients update, since the learning process converges and gets closer to the optimum of optimization problem (\ref{eq:original_problem_general}), to reach a behavior similar to the one of synchronous FL. Indeed, for Theorem \ref{theo:relaxed_sufficient_conditions} to hold, we only need the same optimization time rounds $\Delta t$ over a window $W$ 

\subsection{Generalization}
\label{subsec:generalization}

Coupled with the theoretical method developed in \cite{FedNova}, the proof of Theorem \ref{theo:convergence_convex} can account for FL regularization methods \citep{FedProx, FedDane, FedDyn}, other SGD solvers \citep{Adam, AdaGrad,pmlr-v89-li19c, OnTheLienarSpeedUp, Yu_Yang_Zhu_2019, haddapour2019trading}, and/or gradient compression/quantization \citep{FedPaq, QSparse, Atomo, koloskova2020decentralized}. 

We also note that Theorem \ref{theo:relaxed_sufficient_conditions} can be applied to other distributed optimization schemes using different waiting time policy $\Delta t^n$. With \textsc{FedBuff} \citep{FedBuff}, the server waits for $m$ client updates to create the new global model. The server then communicates to these clients the new global model, while the other clients keep performing local work on the global model they received.

In this section, the sufficient conditions of Theorem \ref{theo:relaxed_sufficient_conditions} regarding the expected aggregation weights $q_i(n)$ were applied to obtain proper aggregation weight $d_i$. We keep identical clients local learning rate $\eta_l$ and amount of local work $K$. 
We could instead get the close-form of a client specific learning rate $\eta_l(i)$ or amount of local work $K(i)$ using the gradient formalization of \cite{FedNova}.

\section{Experiments}\label{sec:experiments}

In this section, we experimentally demonstrate the theoretical claims of Section \ref{sec:convergence} and \ref{sec:applications}.
We first introduce the information needed to understand how the experiments are run in Section \ref{subsec:setting}. Finally, in Section \ref{subsec:exp_results}, we provide our experiments and their interpretation. 

\subsection{Experimental Setting}\label{subsec:setting}

We introduce in this subsection the dataset and the predictive models used for federated optimization, the hardware scenarios proposed to simulate hardware heterogeneity, the clients aggregation weights strategies, and how the different hyperparameters are set.

\textbf{Optimization Problems.}
We consider learning a predictive model for optimization problem (\ref{eq:original_problem_general}) where clients have identical importance ($p_i = 1/M$) based on the following datasets with their associated learning scenarios. 
\begin{itemize}
	
	
	\item \textbf{MNIST} \citep{MNIST} and \textbf{MNIST-shard}. MNIST is a dataset of 28x28 pixel grayscale images of handwritten single digits between 0 and 9 composed of 60 000 training and 10 000 testing samples split equally among the clients. 
	We use a logistic regression to predict the images class. Clients are randomly allocated digits to match their number of samples. 
	With MNIST-shard, we split instead data samples among clients using a Dirichlet distribution of parameter 0.1, i.e. $Dir(0.1)$.
	Therefore, with MNIST and MNIST-shard, we evaluate our theory on a convex optimization problem.
	
	
	
	\item \textbf{CIFAR-10} \citep{CIFAR-10}. The dataset consists of 10 classes of 32x32 images with three RGB channels.  There are 50000 training and 10000 testing examples. The model architecture was taken from \citep{FedAvg} which consists of two convolutional layers and a linear transformation layer to produce logits. 
	Clients get the same amount of samples but their percentage for each class vary and is determined with a Dirichlet distribution of parameter 0.1, i.e. $Dir(0.1)$ \citep{FL_and_CIFAR_dir}. 
	
	\item \textbf{Shakespeare} \citep{Leaf}. We study a LSTM model for next character prediction on the dataset of \textit{The Complete Works of William Shakespeare}. The model architecture is composed of a two-layer LSTM classifier containing 100 hidden units with an 8 dimensional embedding layer taken from \citep{FedProx}. The model takes as an input a sequence of 80 characters, embeds each of the characters into a learned 8-dimensional space and outputs one character per training sample after 2 LSTM layers and a fully connected one.
\end{itemize}

\textbf{Hardware Scenarios.}
In the following experimental scenarios, clients computation time are obtained according to the time policy $FX$. We consider that clients have fixed update times that can be up to X\% slower than the fastest client. Clients computation time are uniformly distributed from the upper to the lower bound.
Clients have thus identical hardware with $F0$. To simulate heterogeneous clients hardware, we consider the time scenario $F80$.

\textbf{Clients Aggregation Weights.}
To compare asynchronous FL with and without the close-form of $d_i$ provided in Section \ref{sec:applications}, we introduce \textsc{Identical} where $d_i =1$ for every client regardless of the time scenario $FX$, and \textsc{Time-based} where $d_i$ satisfies equation (\ref{eq:aggreg_weights_async}) derived in Section \ref{sec:applications}.

\textbf{Hyperparameters.}
Unless specified otherwise, we consider a global learning rate $\eta_g = 1$. We finetune the local learning rate $\eta_l$ with values ranging from $10^{-5}$ to $1$. We consider a batch size $B=64$ for every dataset. 
We report mean and standard deviation on 5 random seeds. 
Every comparison of \textsc{Identical} with \textsc{Time-Based} is done using the same local learning rate. We give an advantage to \textsc{Identical} by finetuning the learning rate on this clients aggregation weight scenario.

\subsection{Experimental Results}\label{subsec:exp_results}

We experimentally show that asynchronous FL has better performances with \textsc{Time-based} than with \textsc{Identical}, and thus we demonstrate the correctness of Theorem \ref{theo:relaxed_sufficient_conditions} with Figure \ref{fig:improvement_conv_async} in Section \ref{subsubsec:convergence_async}. We however show in Figure \ref{fig:amount_of_SGD} that \textsc{Time-based} is less stable than \textsc{Identical} to the change in amount of local work $K$. Finally, we compare synchronous \textsc{FedAvg} and asynchronous \textsc{FedAvg} in Figure \ref{fig:all_opt_scheme}. 

\subsubsection{Impact of the Clients Aggregation Weights on Asynchronous \textsc{FedAvg}}\label{subsubsec:convergence_async}

\begin{figure}
	\includegraphics[width=\linewidth]{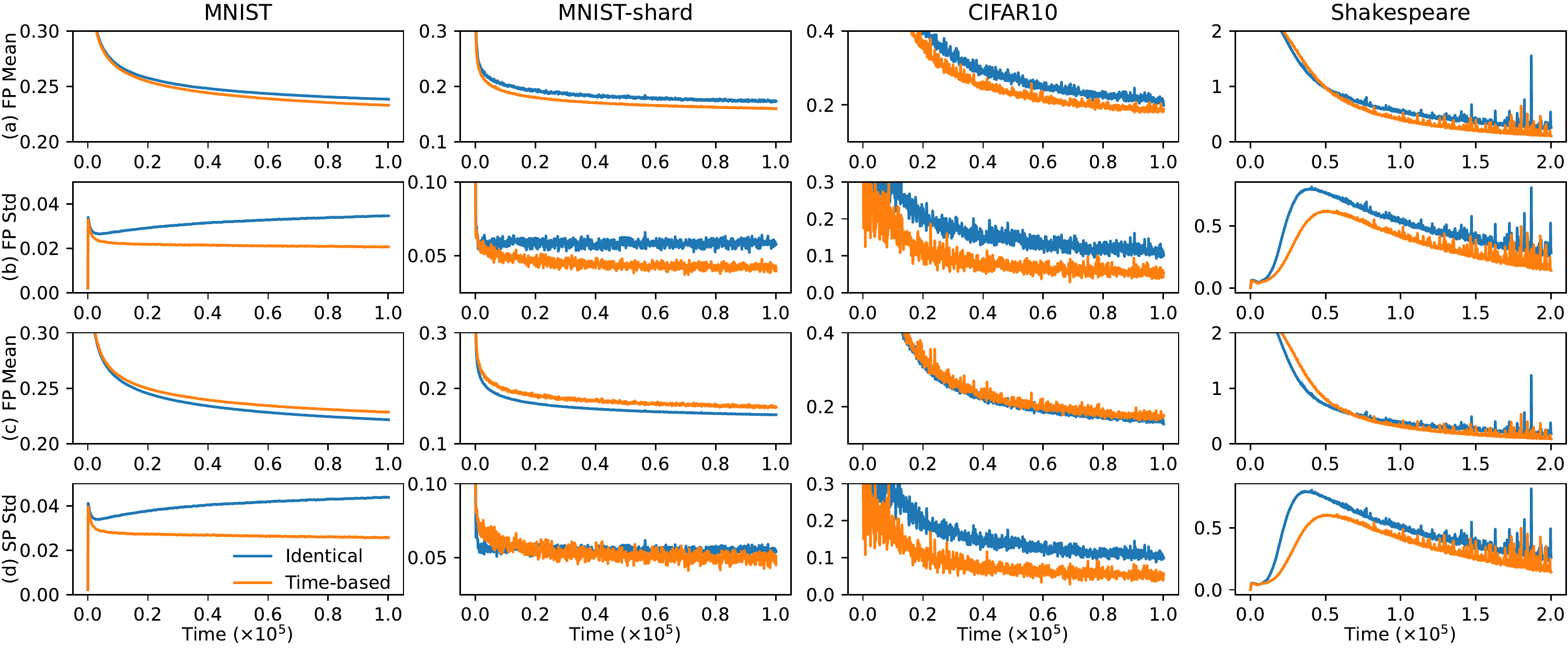}
	\caption{We consider the loss evolution over time of federated problem (\ref{eq:original_problem_general}) (FP) and surrogate problem (\ref{eq:surrogate_problem}) (SP) for MNIST, MNIST-shard, CIFAR10, and Shakespeare; and the respective standard deviation of the loss over clients in (b) and (d).
	We consider $M=10$ for a time scenario $F80$ with $K=1$. 
	}
	\label{fig:improvement_conv_async}
\end{figure}

Figure \ref{fig:improvement_conv_async}(a) experimentally shows the interest of coupling asynchronous FL with \textsc{Time-based} instead of \textsc{Identical} for different applications (MNIST, MNIST-shard, CIFAR10, and Shakespeare). The learnt model with \textsc{Time-based} has better minima on the federated problem (\ref{eq:original_problem_general}). 
In addition, Figure \ref{fig:improvement_conv_async}(b) shows that losses across clients are more homogeneous with \textsc{Time-based}, resulting in generally lower standard deviations. 

Focusing on MNIST and MNIST-shard, we see the impact of data heterogeneity on the learnt model performances. With \textsc{Identical}, asynchronous FL converges to a suboptimum point and the differences between the learnt model losses is twice as large for MNIST-shard than for MNIST, Figure \ref{fig:improvement_conv_async}(a). 
Figure \ref{fig:improvement_conv_async}(b) shows a similar result concerning the clients loss heterogeneity.
Therefore, data heterogeneity degrades the suboptimum loss and cannot be ignored in asynchronous FL applications. Indeed, \textsc{Identical} and \textsc{Time-based} curves are significantly different even for the simplest application on MNIST, where the dataset is uniformly distributed across $M=10$ clients.
Hence, the assumption of identical data distributions should generally not be made and the aggregation scheme \textsc{Time-based} should be used instead for any asynchronous FL (or DL). 

With Figure \ref{fig:improvement_conv_async}(c), we can also appreciate the performances of the learning procedure on the surrogate problem (\ref{eq:surrogate_problem}) based on the clients computation times $T_i$. Due to clients hardware heterogeneity, in the scenario $F80$, clients communicate with the server up to 5 more times than the slowest one. \textsc{Time-based} balances this amount of updates disparity across clients. As a result, \textsc{Identical} has better performances than \textsc{Time-based} on the surrogate problem (\ref{eq:surrogate_problem}) for MNIST, MNIST-shard, and CIFAR10, while for Shakespeare, \textsc{Time-based} shows better performances. We attribute this fact to the depth of the predictive model enabling overfitting. As such, \textsc{Time-based} outperforms \textsc{Identical} on the federated problem (\ref{eq:original_problem_general}), Figure \ref{fig:improvement_conv_async}(a), while preventing catastrophic forgetting and thus leading to better losses on fast clients, Figure \ref{fig:improvement_conv_async}(c). Finally, Figure \ref{fig:improvement_conv_async}(d) shows that in addition, the weighted standard deviation of the surrogate loss is always worse for \textsc{Identical}.

\subsubsection{Impact of the amount of local work on asynchronous FL convergence}
\label{subsubsec:async_nSGD}

\begin{figure}
	\centering
	\includegraphics[width=0.5 \linewidth]{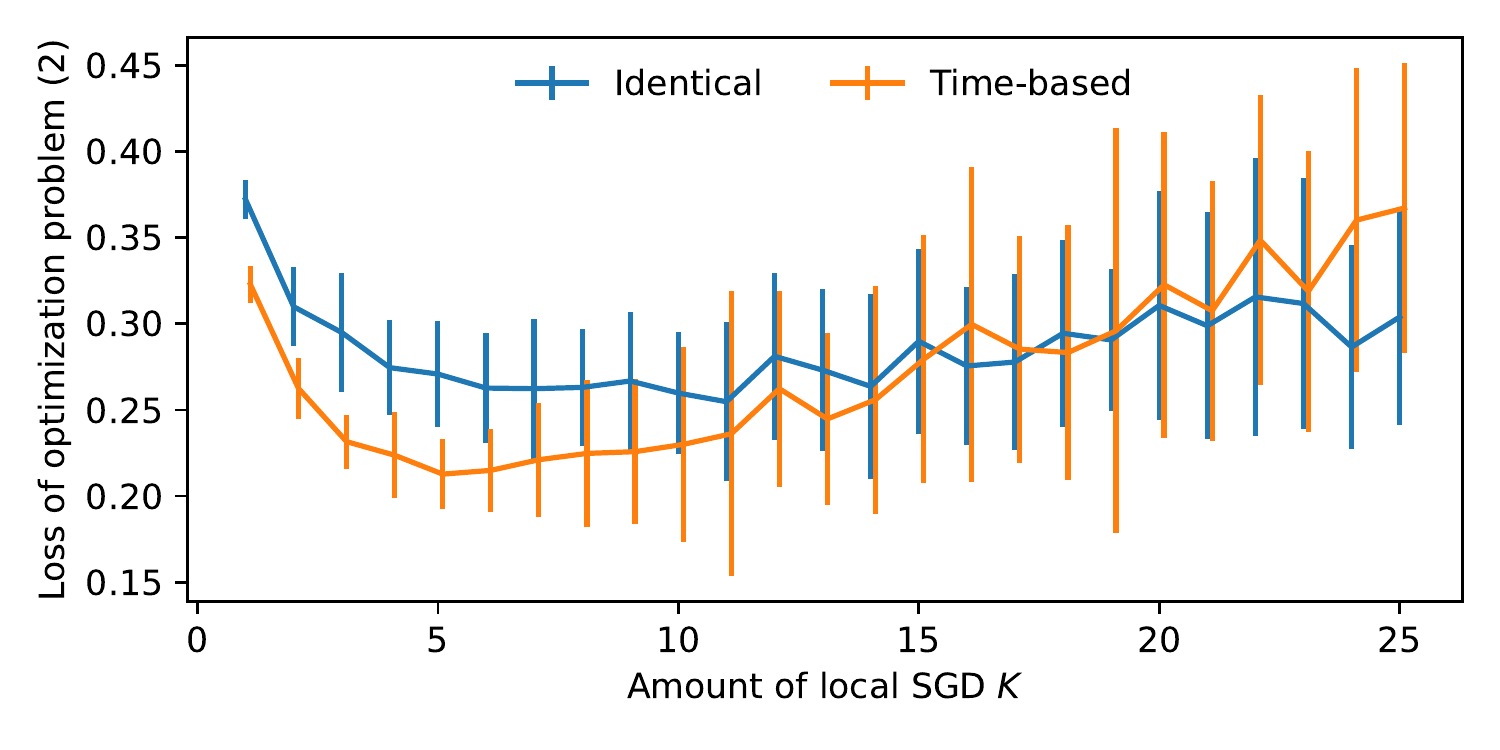}
	\caption{Evolution of the loss of optimization problem (\ref{eq:original_problem_general}) for CIFAR10 with $M=10$, time scenario $F80$, $\eta_g = 1$, $\eta_l = 0.005$, and varying amount of local work $K$ ranging from 1 to 25.}
	\label{fig:amount_of_SGD}
\end{figure}

With Figure \ref{fig:amount_of_SGD}, we consider the impact of the amount of local work on the convergence speed of Asynchronous FL with CIFAR10, time scenario $F80$, and $M=10$ clients. For every simulation, we consider $\eta_l = 0.0005$, the optimal local learning rate for $K=1$ with \textsc{Identical}. The server aggregates the clients contribution over $T = 25000$ units of time, and we report in Figure \ref{fig:amount_of_SGD} mean and standard deviation over the 5\% last server optimization rounds of the loss of the optimization problem (\ref{eq:original_problem_general}).

Figure \ref{fig:amount_of_SGD} shows that increasing the amount of local work $K$ first decreases the loss of optimization problem (\ref{eq:original_problem_general}) evaluated on the expected learnt model before increasing it. This point justifies asking to clients to perform $K>1$ SGDs but requires proper finetuning of the amount of local work $K$. In particular, we notice that the variance strictly increases with $K$, which shows that the learning procedure becomes less stable.

These behaviors of asynchronous FL are due to the disparity between clients contributions induced by clients data heterogeneity, and can only be mitigated with a smaller global learning rate $\eta_g$, local learning rate $\eta_l$, and amount of local work $K$. 
Figure \ref{fig:amount_of_SGD} shows that \textsc{Time-based} is however more sensitive to an increase in amount of local work $K$ than \textsc{Identical}. \textsc{Time-based} is associated with higher variance after $K=8$, and higher mean after $K=15$, while \textsc{Identical} has  very similar mean and standard deviation from $K=8$ to $K=16$. 

This difference in convergence behavior is due to the FL aggregation scheme (\ref{eq:aggreg_SCAFFOLD_general}), and to the difference between the clients aggregation weights $d_i$ for \textsc{Identical} and \textsc{Time-based}. We have indeed $d_i =1$ for every client with \textsc{Identical}, while with \textsc{time-based} fast clients are given lower aggregation weights $d_i<1$, and slow clients higher weights $d_i>1$. 
Therefore, whenever a slow client contributes, the new global model is more perturbed with \textsc{Time-based} than with \textsc{Identical}, which makes \textsc{Time-based} convergence speed more sensitive to a small change in the choice of $K$ and other hyperparameters. 
This point can also be noticed in Figure \ref{fig:improvement_conv_async}(a) where \textsc{Identical} first converges faster than \textsc{Time-based}. Still, \textsc{Identical} converges to a suboptimum. 

\textsc{Identical} outperforms \textsc{Time-based} in Figure \ref{fig:amount_of_SGD} because we consider $\eta_l = 0.0005$.
We note that considering our time budget $T$, doing a grid search for $\eta_l$ would always provide a learnt model with better optimization loss for \textsc{Time-based}.

%
%

%
%
%

\subsubsection{Partial Asynchronicity with \textsc{FedFix}}\label{subsubsec:partial_asynchronicity}

\begin{figure}
	\centering
	\includegraphics[width=\linewidth]{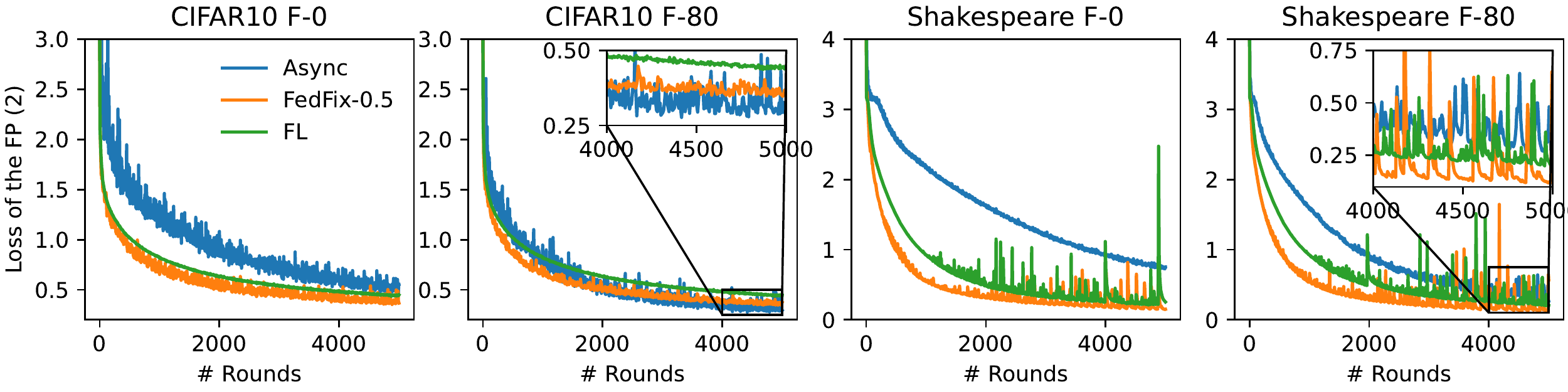}
	\caption{Evolution of federated problem (\ref{eq:original_problem_general}) loss for CIFAR10 and Shakespeare and time scenario $F0$ and $F80$, with $M=20$ (a) and $M=50$ (b). We consider $\eta_g = 1$, $K = 10$, and $\Delta t=0.5$ for \textsc{FedFix}.}
	\label{fig:all_opt_scheme}
\end{figure}

The theory derived in Section \ref{sec:convergence} can be applied to asynchronous FL but also synchronous FL, \textsc{FedAvg}, and other asynchronous FL schemes like \textsc{FedFix} (Section \ref{sec:applications}). 
We show with Figure \ref{fig:all_opt_scheme} that allowing asynchronicity does not necessarily provide faster learning processes, e.g. comparison between synchronous and asynchronous \textsc{FedAvg} above, but \textsc{FedFix} always outperforms \textsc{FedAvg} by balancing convergence speed and stability.

With a small enough learning rate $\eta_l$, asynchronous \textsc{FedAvg} outperforms \textsc{FedFix}, which outperforms synchronous FL (see Figure \ref{app:fig:explore_CIFAR10_20} and \ref{app:fig:explore_Shakespeare_20} in Appendix \ref{app:sec:experiments}). Indeed, in this case, global models change slowly and we can consider that the server receives contributions with no gradient delay. As such, the learning procedure including the most serial contributions in the global model is the fastest. 
However, in the other cases, the learning rate $\eta_l$ does not mitigate the discrepancy between clients update, which slows down convergence for asynchronous FL, and can even prevent it.

Identifying the fastest optimization scheme must be done by comparing optimization schemes based on their best local learning rate $\eta_l$ (Figure \ref{fig:all_opt_scheme}). 
Synchronous FL always outperforms asynchronous FL when clients have heterogeneous hardware ($F0$). Even with heterogeneous hardware ($F80$), synchronous FL can outperform asynchronous FL (Shakespeare). Indeed, the server needs to reduce its amount of aggregations to balance convergence speed and convergence stability. We see that \textsc{FedFix-0.5} provides this trade-off and outperforms synchronous FL in every scenario. 

We note that, even for synchronous FL, FL convergence is not monotonous. Indeed, for synchronous FL to have a better convergence speed than asynchronous FL, the server needs to consider a high local learning rate leading to convergence instability. Figure \ref{fig:all_opt_scheme} shows this instability for Shakespeare and $t>4000$, and Figure \ref{app:fig:explore_CIFAR10_20} to \ref{app:fig:explore_Shakespeare_20} in Appendix \ref{app:sec:experiments} provides the evolution of this instability as the learning rate $\eta_l$ increases.

We note that even when clients have homogeneous hardware ($F0$), \textsc{FedFix} outperforms synchronous FL. This can be explained by the close-form of \textsc{FedFix} weights $d_i$, equation (\ref{eq:aggreg_weights_hybrid}), which accounts for server aggregations where no client participates. As a result, \textsc{FedFix-0.5} behaves as asynchronous FL but with an higher server learning rate $\eta_g =2$ which provides faster convergence.

\section{Discussion}

This work introduces equation (\ref{eq:aggreg_SCAFFOLD_general}) which generalizes the expression of \textsc{FedAvg} aggregation scheme by introducing stochastic aggregation weights $\omega_i(n)$ to account for asynchronous client updates. With a simple assumption for clients aggregation weights covariance, Assumption \ref{ass:clients_covariance}, we prove the convergence of FL schemes satisfying equation (\ref{eq:aggreg_SCAFFOLD_general}).
A similar aggregation scheme has been derived in \cite{OnTheImpact} for unbiased client sampling, which this work generalizes. 
In addition, we show that aggregation scheme (\ref{eq:aggreg_SCAFFOLD_general}) and Assumption \ref{ass:clients_covariance} are satisfied by asynchronous FL, \textsc{FedFix}, and \textsc{FedBuff}, Section \ref{sec:applications}. 
Finally, we assume fixed clients update time $T_i$ such that 
we can consider $d_i(n ) = d_i$, and 
give in Section \ref{sec:applications} its close-form 
to ensure any FL optimization scheme converges to the optimum of problem (\ref{eq:original_problem_general}). Our work remains relevant for applications with $d_i(n) = d_i$ but we let the specific derivations to the reader.

This work shows theoretically and experimentally that asynchronous \textsc{FedAvg} does not always outperform its synchronous counterpart.
By creating the new global model with the contribution of only one client, asynchronous \textsc{FedAvg} convergence speed is very sensitive to the choice of learning rate and amount of local work $K$. These two hyperparameters need to be fine-tuned to properly balance convergence speed and stability. 
Due to the hardware constraints inherent to the FL setting, fine-tuning is a challenging step for FL and is not necessarily feasible. Therefore, we proposed \textsc{FedFix}, an FL algorithm where the server, after a fixed amount of time, creates the new global model with the contribution of all the participating clients. We prove the convergence of \textsc{FedFix} with our theoretical framework, and experimentally demonstrate its improvement over \textsc{FedAvg} in all the considered scenarios.



\appendix

\section{Proof of Theorem \ref{theo:convergence_convex}} 
\label{app:sec:convergence_convex}

We first provide in Section \ref{app:subsec:inequalities} the basic inequalities used in our proofs, and in Section \ref{app:subsec:notation} the basic notations used to provide clearer proofs.

\subsection{Basic Inequalities}\label{app:subsec:inequalities}
We provide the following basic inequalities used in our proofs.

Let us consider $f$ a $L$-Lipschitz smooth and convex function with optimum $\vx^*$. 
For any vector $\vx$ and $\vy$, we have
\begin{equation}
\norm{\nabla f (\vx)}^2
\le 2 L [f(\vx) - f(\vx^*)]
,
\text{ and }
\norm{\nabla f (\vx) - \nabla f (\vy) }^2
\le L^2 \norm{\vx - \vy}^2
\label{app:eq:lipschitz}
.
\end{equation}

Let us consider $g$ a convex function and $d$ vectors $\{\vx_k\}$ each with importance $p_k$ such that $\sum_{ k = 1 }^d p_k = 1$. With Jensen inequality, we have
\begin{equation}
	g(\sum_{ k = 1 }^d p_k \vx_k) 
	\le \sum_{k=1}^d p_k g(\vx_k)
	.
\end{equation}

Let us consider the random variable $X$, we have
\begin{equation}
	\E{\norm{X - \E{X}}^2}
	\le \E{\norm{X}^2}
	.
\end{equation}

\begin{table}
	\caption{Common Notation Summary (addition to Table \ref{tab:aggregation_weights}).}
	\label{tab:parameters}
	\centering
	\begin{tabular}{cc}
		\hline
		Symbol     & Description    \\
		\hline
		$M$ & Number of clients. \\
		$K$ & Number of local SGD.  \\
		$\eta_g, \eta_l$ & Global/Local learning rate.\\
		$\tilde{\eta}$ & Effective learning rate, $\tilde{\eta} = \eta_l \eta_g$.\\
		
		$\vtheta^n$ & Global model at server iteration $n$.  \\
		$\vtheta_i^{n+1}$ & Local update of client $i$ on model $\vtheta^n$.\\
		$\vtheta^*$, $\vtheta_i^*$ & Optimum of the federated problem (\ref{eq:original_problem_general})/client $i$. \\
		$\vtheta^{(n, k)}$, $\vtheta_i^{(n, k)}$ & Global/Local update after $k$ SGD on global model $\vtheta^n$.\\
		$\alpha$ & Covariance parameter.\\
		$\beta$ & Defined in Theorem \ref{theo:convergence_convex}.\\
		
		$\Lcal(\cdot), \Lcal_i(\cdot)$ & Federated/local loss function.\\
		$g_i(\cdot)$ & SG. We have $\EE{\xi_i}{g_i(\cdot)} = \nabla \Lcal_i(\cdot)$ with Assumption \ref{ass:unbiased}.\\
		$\vxi_i$ & Random batch of samples from client $i$ of size $B$.\\

		$L$ & Lipschitz smoothness parameter, Assumption \ref{ass:smoothness}.\\
		
		$T_i$ & Computation time of client $i$. \\
		$t^n$ & Time at aggregation $n$. \\
		$T_i^n$ & Remaining computation time of client $i$ at time $t^n$. \\
		$\Delta t^n$ & Time elapsed between two server aggregations. \\
		$\rho_i(n)$ & Last index at which a client $i$ received its global model. \\
		
		$\rho$ & Highest sum of aggregation weights, i.e. $\rho \coloneqq \max \left(1, q(n)\right)$\\

		\hline
	\end{tabular}
\end{table}

\subsection{Additional Notation}\label{app:subsec:notation}
In Table \ref{tab:aggregation_weights}, we synthesize the different random variables associated to the clients aggregation weights. 
In Table \ref{tab:parameters}, we synthesize the remaining random variables.

We introduce the following notations to provide clear and compact proofs.
Whenever considering a function $f(n, k)$, we define $f(n) \coloneqq 1/K \sum_{ k = 0 }^{K-1} f(n , k)$, and $\bar{f}(N) \coloneqq 1/N \sum_{ n = 0 }^{N-1} f(n)$.
We introduce the following quantities
\begin{align}
D(\vx, n, k)
\coloneqq
\E{ \inner { \sum_{i=1}^M q_i(n) \nabla \Lcal_i (\vtheta_i^{\rho_i(n), k})}{\vtheta^{n,k} - \vx}}
,
&&
Q(n)
\coloneqq \E{\norm{\vtheta^{n+1, 0} - \vtheta^{n ,0}}^2} 
,
\end{align}
\begin{align}
R(n , k) 
\coloneqq \E{\norm{\sum_{ i = 1 }^M \tilde{q}_i(n) g_i(\vtheta_{i}^{\rho_i(n), k})}^2}
,
&&
S(n , k) 
\coloneqq \sum_{ i = 1 }^M \tilde{q}_i(n) \E{\norm{g_i(\vtheta_{i}^{\rho_i(n), k})}^2}
,
\end{align}
\begin{align}
Z(n , k)
= \Lcal^n(\vtheta^{n, k}) - \Lcal^n(\bar{\vtheta}^n)
,
&&
\Delta(n, k) 
\coloneqq \E{\norm{\vtheta^{n, k+ 1} - \vx}^2} - \E{ \norm{\vtheta^{n, k} - \vx}^2 }
,
\end{align} 
\begin{align}
\phi(n, k )
\coloneqq \sum_{ i = 1 }^M \tilde{q}_i(n) \E{\norm{\vtheta_i^{\rho_i(n), k} - \vtheta^{n,k}}^2 }
,
&&
\sigma_1(n) 
\coloneqq \sum_{ i = 1 }^M \tilde{q}_i(n)\E{\norm{ \nabla \Lcal_i(\bar{\vtheta}^n, \vxi_i)}^2}
,
\end{align}
\begin{align}
\sigma_2(n) 
\coloneqq \sum_{ i = 1 }^M \tilde{q}_i^2(n)\E{\norm{ \nabla\Lcal_i(\bar{\vtheta}^n, \vxi_i)}^2}
,
\text{ and }
\Xi(n , k)
= \Lcal^n(\vtheta^{n, k}) - \Lcal^n(\vx)
.
\end{align}
Finally, we define $g_i(\vy) = \nabla \Lcal_i(\vy, \vxi_i)$ the SG of client $i$ evaluated on model parameters $\vy$ and batch $\vxi_i$. We will thus write $g_i(\vtheta_{i, k}^{\rho_i(n)})$ instead of $\nabla \Lcal_i(\vtheta_{i, k}^{\rho_i(n)}, \vxi_{i, k}^{\rho_i(n)})$.


\subsection{Useful Lemmas}\label{app:subsec:useful_lemmas}
\begin{lemma}\label{lem:decompo_Xi}
	Let us consider $n$ vectors $\vx_i, ..., \vx_n$ and assume Assumption \ref{ass:clients_covariance}. We have
	\begin{equation}
	\EE{S_n}{\norm{\sum_{i=1}^M\omega_i(n)\vx_i}^2} 
	= \sum_{i=1}^M \gamma_i(n) \norm{\vx_i}^2
	+ \alpha \norm{\sum_{i=1}^M q_i(n) \vx_i}^2
	,
	\end{equation}
	where $\gamma_i(n)
	= \EE{S_n}{\omega_i^2(n)} - \alpha q_i^2(n)
	\ge 0$, and $\gamma_i(n) \le \beta q_i(n)$ with $\beta \coloneqq \max\{d_i(n) - \alpha q_i(n)\}$.
\end{lemma}

\begin{proof}
	\begin{align}
	\EE{S_n}{\norm{\sum_{i=1}^M\omega_i(n)\vx_i}^2} 
	& = \sum_{i=1}^M \EE{S_n}{\omega_i^2(n)} \norm{\vx_i}^2
	+ \sum_{i=1}^M\sum_{\substack{j=1\\ j\neq i}}^M\EE{S_n}{\omega_i(n)\omega_j(n)}\inner{\vx_i}{\vx_j}
	\\
	& = \sum_{i=1}^M \EE{S_n}{\omega_i^2(n)} \norm{\vx_i}^2
	+ \sum_{i=1}^M \sum_{\substack{j=1\\ j\neq i}}^M \alpha q_i(n) q_j(n) \inner{\vx_i}{\vx_j}
	\label{app:eq:AA1}
	.
	\end{align}
	In addition, we have
	\begin{equation}
	\sum_{i=1}^M\sum_{\substack{j=1\\ j\neq i}}^M\inner{q_i(n) \vx_i}{q_j(n) \vx_j} 
	= \norm{\sum_{i=1}^M q_i(n) \vx_i}^2 - \sum_{i=1}^M q_i^2(n) \norm{\vx_i}^2
	\label{app:eq:AA3}
	.
	\end{equation}
	Substituting equation (\refeq{app:eq:AA3}) in equation (\refeq{app:eq:AA1}) completes the first claim.
	
	Considering that $\EE{S_n}{\omega_i^2(n)} = \VAR{\omega_i(n)} + q_i^2(n) \ge q_i^2(n)$ and $\alpha \le 1$, we have $\gamma_i(n) \ge 0$ which completes the second claim.
	
	Finally, the third claim follows directly from the close-form of the clients aggregation weights, equation (\refeq{eq:agg_weights}).

	\textbf{Remark.}
	We can also provide the following lower bound for equation (\refeq{app:eq:AA3}) using Jensen inequality
	\begin{equation}
	\sum_{i=1}^M\sum_{\substack{j=1\\ j\neq i}}^M\inner{q_i(n) \vx_i}{q_j(n) \vx_j} 
	\ge \norm{\sum_{i=1}^M q_i(n) \vx_i}^2 - \frac{\max q_i(n)}{q(n)}\norm{\sum_{i=1}^M q_i(n) \vx_i}^2
	\ge 0
	\label{app:eq:AA4}
	.
	\end{equation}
	Therefore, $\EE{S_n}{\norm{\sum_{i=1}^M\omega_i(n)\vx_i}^2} $ is linearly proportional to $\alpha$.
	
\end{proof}

\begin{lemma}\label{lemma:one_SGD}
	Under Assumption \ref{ass:clients_covariance}, the following equation holds for any vector $\vx$:
	\begin{align}
	\Delta(n)
	&\le - 2 \tilde{\eta} D(\vx, n)
	+ \tilde{\eta}^2 \alpha q^2(n) R(n)
	+ \tilde{\eta}^2 \beta q(n) S(n)
	.
	\end{align}

\end{lemma}

\begin{proof}
	We consider $S_n$, the set of participating clients at optimization round $n$, i.e. $S_n = \{n: T_i^n \le \Delta t^n\}$. We have
	\begin{align}
	\EE{S_n}{\norm{\vtheta^{n, k+1} - \vtheta^*}^2} 
	&= \EE{S_n}{\norm{(\vtheta^{n, k+1} - \vtheta^{n ,k}) + (\vtheta^{n, k} - \vtheta^*)}^2}\\
	&= \norm{\vtheta^{n, k} - \vtheta^*}^2 
	+ 2 \inner{\EE{S_n}{\vtheta^{n, k+1} - \vtheta^{n,k}}}{\vtheta^{n ,k} - \vtheta^*} \nonumber\\
	& + \EE{S_n}{\norm{\vtheta^{n, k+1} - \vtheta^{n, k}}^2} 
	\label{app:eq:decomposition}
	.
	\end{align}
	By construction, we have 
	$\vtheta^{n, k+1} - \vtheta^{n , k} 
	=  - \tilde{\eta} \sum_{i=1}^M \omega_i(n) g_i(\vtheta_i^{\rho_i(n) , k})$.
	Taking the expectation over $S_n$, we can simplify the second term of equation (\refeq{app:eq:decomposition}) with 
	$\EE{S_n}{\vtheta^{n, k+1} - \vtheta^{n , k}} 
	= - \tilde{\eta} \sum_{i=1}^M q_i(n) g_i(\vtheta_i^{\rho_i(n) , k})$. 
	Finally, using Lemma \ref{lem:decompo_Xi}, we can bound the third term.
	Therefore, we have
	\begin{align}
	\EE{S_n}{\norm{\vtheta^{n, k+1} - \vtheta^*}^2} 
	&= \norm{\vtheta^{n, k} - \vtheta^*}^2 
	+ 2 \tilde{\eta} \inner{\sum_{i=1}^M q_i(n)  g_i(\vtheta_i^{\rho_i(n) , k})}{\vtheta^n - \vtheta^*} 
	\nonumber\\
	&+ \tilde{\eta}^2 \sum_{i=1}^M \gamma_i(n)  \norm{g_i(\vtheta_i^{\rho_i(n) , k})}^2 
	+ \tilde{\eta}^2 \alpha  \norm{\sum_{i=1}^M  q_i(n) g_i(\vtheta_i^{\rho_i(n) , k})}^2 
	\label{app:eq:A1}
	.
	\end{align}
	Considering $\gamma_i(n) \le \beta q_i(n)$, taking the expected value over the iteration random batches $\vxi^{\rho_i(n), k}$, and finally taking the expected value over the remaining random variables gives
	\begin{align}
	\Delta(n, k)
	&\le - 2 \tilde{\eta} D(\vx, n, k)
	+ \tilde{\eta}^2 \alpha q^2(n) R(n ,k)
	+ \tilde{\eta}^2 \beta q(n) S(n ,k)
	.\label{eq:one_server_iteration}
	\end{align}
	Taking the mean over $K$ completes the proof.
	
\end{proof}
\begin{lemma}\label{lem:phi_nk}
	
	Under Assumption \ref{ass:unbiased} and \ref{ass:smoothness}, and $D \coloneqq 6 \eta_l^2 (K-1)^2 L^2 \le 1/2$, we have
	\begin{align}
	\phi(n)
	&\le  4 q(n) \tau \sum_{s=1}^\tau  Q(n - s) 
	+ 4 D \frac{1}{L} q^{-1}(n) Z(n)
	+ 6 \eta_l^2 (K-1)^2 \sigma_1(n)
	,
	\end{align}
	\begin{align}
	\text{and } S(n)
	\le 12 q(n) L^2 \tau \sum_{ s = 1 }^\tau Q(n - s)
	+ 12 L q^{-1}(n) Z(n)
	+ 6 \sigma_1(n)
	.
	\end{align}
\end{lemma}

\begin{proof}
	Let us decompose the difference $\vtheta_i^{\rho_i(n), k} - \vtheta^{n,k}$ as
	\begin{equation}
		\vtheta_i^{\rho_i(n), k} - \vtheta^{n,k}
		= \left[\vtheta^{\rho_i(n)} - \eta_l\sum_{ l = 0 }^{k-1} g_i(\vtheta_i^{\rho_i(n), l})\right]
		- \left[\vtheta^n - \eta_l\sum_{ l = 0 }^{k-1}\sum_{ i = 1 }^M \tilde{q}_i(n) g_i(\vtheta_i^{\rho_i(n), l})\right]
	.
	\end{equation}
	Using Jensen inequality, we split the difference between the global models and the one between the gradients to get
	\begin{equation}
	\norm{\vtheta_i^{\rho_i(n), k} - \vtheta^{n,k}}^2
	\le  2 \norm{\vtheta^{\rho_i(n)} - \vtheta^n}^2
	+ 2 \eta_l^2 k \sum_{ l = 0 }^{k-1}\norm{ g_i(\vtheta_i^{\rho_i(n), l})
	- \sum_{ i = 1 }^M \tilde{q}_i(n) g_i(\vtheta_i^{\rho_i(n), l})}^2
	.
	\label{app:eq:O1}
	\end{equation}
	Therefore, by taking the expectations of equation (\refeq{app:eq:O1}) and summing over $M$ gives
	\begin{align}
	\phi(n , k)
	&\le  2 \sum_{ i = 1 }^M \tilde{q}_i(n) \E{\norm{\vtheta^{\rho_i(n)} - \vtheta^n}^2}
	\nonumber\\
	& + 2 \eta_l^2 k \sum_{ l = 0 }^{ k - 1} \sum_{ i = 1 }^M \tilde{q}_i(n)\E{\norm{ g_i(\vtheta_i^{\rho_i(n), l})
		- \sum_{ i = 1 }^M \tilde{q}_i(n) g_i(\vtheta_i^{\rho_i(n), l})}^2}
	\\
	&\le  2 \sum_{ i = 1 }^M \tilde{q}_i(n) \E{\norm{\vtheta^{\rho_i(n)} - \vtheta^n}^2} 
	+ 2 \eta_l^2 k \sum_{ l = 0 }^{k-1}\sum_{ i = 1 }^M \tilde{q}_i(n)\E{\norm{ g_i(\vtheta_i^{\rho_i(n), l})}^2}
	\label{app:eq:O2}
	,
	\end{align}
	where we see that $S(n, l)$ appears in the second term of equation (\refeq{app:eq:O2}). We consider now bounding $S(n , k)$, and first note that a stochastic gradient can be bounded as follow
	\begin{align}
	\E{\norm{g_i(\vtheta_i^{\rho_i(n), k})}^2}
	& \le 3 \E{\norm{\nabla \Lcal_i(\vtheta_i^{\rho_i(n),k}, \vxi_{i, k}^{\rho_i(n)}) - \nabla \Lcal_i(\vtheta^{n,k}, \vxi_{i, k}^{\rho_i(n)}) }^2} 
	\nonumber\\
	&+ 3 \E{\norm{\nabla \Lcal_i(\vtheta^{n,k}, \vxi_i) - \nabla \Lcal_i(\bar{\vtheta}^n, \vxi_i) }^2} 
	+ 3 \E{\norm{\nabla \Lcal_i(\bar{\vtheta}^n, \vxi_i) }^2} 
	.
	\label{app:eq:F_0}
	\end{align}
	When summing equation (\refeq{app:eq:F_0}) over $M$, and considering the clients loss functions Lipschitz smoothness, Assumption \ref{ass:smoothness}, we have
	\begin{align}
	S(n, k)
	=	\sum_{ i = 1 }^M \tilde{q}_i(n)\E{\norm{g_i(\vtheta_i^{\rho_i(n), k})}^2}
	\le 3 L^2 \phi(n , k)
	+ 6 L q^{-1}(n)Z(n , k)
	+ 3  \sigma_1(n)
	.
	\label{app:eq:F_11}
	\end{align}
	We also note the following intermediary results 
	\begin{equation}
	\sum_{ k = 0 }^{K-1} k \sum_{ l = 0 }^{k-1}x_l
	\le (K-1)\sum_{ k = 1 }^{K-1} \sum_{ l = 0 }^{k-1}x_l
	\le (K - 1)^2 \sum_{ k = 0 }^{K-2}x_k
	\le (K - 1)^2 \sum_{ k = 0 }^{K-1}x_k
	\label{app:eq:O3}
	.
	\end{equation}
	We substitute equation (\refeq{app:eq:F_11}) in equation (\refeq{app:eq:O2}) such that $D$ appears, take the mean over $K$ to introduce $\phi(n)$ on the two sides of the equation, and use equation (\refeq{app:eq:O3}). We have
	\begin{align}
	\phi(n)
	\le  2 \sum_{ i = 1 }^M \tilde{q}_i(n) \E{\norm{\vtheta^{\rho_i(n)} - \vtheta^n}^2} 
	+ D \phi(n) 
	+ 2 D \frac{1}{L} q^{-1}(n) Z(n)
	+ 6 \eta_l^2 (K-1)^2 \sigma_1(n)
	.
	\label{app:eq:phin}
	\end{align}
	Finally, reminding that $D \le 1/2$, which gives $1 - D \ge 1/2$, and 
	using Assumption \ref{ass:answering_time} to bound $\E{\norm{\vtheta^{\rho_i(n)} - \vtheta^n}^2}$  with Jensen inequality completes the first claim for $\phi(n)$, i.e.
	\begin{equation}
		\E{\norm{\vtheta^{\rho_i(n)} - \vtheta^n}^2}
		\le \tau \sum_{s=1}^\tau \E{\norm{\vtheta^{n - s+1 } - \vtheta^{n - s}}^2}
		= \tau \sum_{s=1}^\tau  Q(n - s)
		.
	\end{equation}
	
	Substituting the close-form of $\phi(n)$ in equation (\refeq{app:eq:F_11}) completes the claim for $S(n , k)$.

\end{proof}

\begin{lemma}\label{lem:Dxn}
	Under Assumption  \ref{ass:strong_convexity} and \ref{ass:unbiased}, we have
	\begin{align}
	-2 D(\vx, n)
	&\le -2 \Xi(n) 
	+ 4 L q(n) \tau \sum_{s=1}^\tau  Q(n - s) 
	+ 4 D Z(n)
	+ 6 \eta_l^2 (K-1)^2 q(n) L \sigma_1(n)
	.
	\end{align}
\end{lemma}

\begin{proof}
	Follows directly from using Lemma 12 in \cite{Khaled2020} on $D(\vx, n, k)$, taking the mean over $K$, and using Lemma \ref{lem:phi_nk} to bound $\phi(n)$ completes the proof.
	

\end{proof}
\begin{lemma}\label{lem:bounding_R}
	Under Assumption \ref{ass:smoothness} and \ref{ass:unbiased}, and 
	considering $D \le 1/2$, we have
	\begin{align}
	R(n)
	&\le 12 L^2 \tau \sum_{s=1}^\tau  Q(n - s) 
	+ 24 L q^{-1}(n) Z(n)
	+ 3 D \sigma_1(n)
	+ 6  \sigma_2(n) 
	.
	\end{align}
	
\end{lemma}

\begin{proof}
	\begin{align}
	R(n, k)
	&\le 3 \E{\norm{\sum_{ i = 1 }^M \tilde{q}_i(n)\left[ g_i(\vtheta_i^{\rho_i(n), k}) - \nabla \Lcal_i(\vtheta^{n, k}, \vxi_{i, k}^{\rho_i(n)}) \right]}^2}
	\nonumber\\
	&+ 3  \E{\norm{\sum_{ i = 1 }^M \tilde{q}_i(n) \left[\nabla \Lcal_i(\vtheta^{n, k}, \vxi_i) - \nabla \Lcal_i(\vtheta^{n, k})\right]}^2}
	\nonumber\\
	& + 3 \E{\norm{\sum_{ i = 1 }^M \tilde{q}_i(n) \nabla \Lcal_i(\vtheta^{n, k})}^2}
	\label{app:eq:S1}
	.
	\end{align}
	We respectively call the three terms of equation (\refeq{app:eq:S1}), $a(n , k)$, $b(n, k)$, and $c(n, k)$.
	Using the local loss functions Lipschitz smoothness, Assumption \ref{ass:smoothness}, and Jensen inequality, we can bound $a(n, k)$ as
	\begin{equation}
	a(n, k)
	\le 3\sum_{ i = 1 }^M \tilde{q}_i(n)\E{\norm{ g_i(\vtheta_i^{\rho_i(n), k}) - \nabla \Lcal_i(\vtheta^{n, k}, \vxi_{i, k}^{\rho_i(n)}) }^2}
	\le 3 L^2 \phi(n, k)
	\label{app:eq:S1_1}
	.
	\end{equation}
	Using the unbiasedness of the gradient estimator, Assumption \ref{ass:unbiased}, and the local loss function  Lipschitz smoothness, Assumption \ref{ass:smoothness}, we can bound $b(n, k)$ as
	\begin{align}
	b(n, k)
	&=  3\sum_{ i = 1 }^M \tilde{q}_i^2(n)\E{\norm{ \nabla\Lcal_i(\vtheta^{n, k}, \vxi_i) - \nabla \Lcal_i(\vtheta^{n, k})}^2}
	\\
	& \le  3\sum_{ i = 1 }^M \tilde{q}^2_i(n)\E{\norm{ \nabla\Lcal_i(\vtheta^{n, k}, \vxi_i) }^2}
	\\
	& \le  6 \sum_{ i = 1 }^M \tilde{q}^2_i(n)
	\left[\E{\norm{ \nabla\Lcal_i(\vtheta^{n, k}, \vxi_i) - \nabla\Lcal_i(\bar{\vtheta}^n, \vxi_i)}^2}
	+ \E{\norm{ \nabla\Lcal_i(\bar{\vtheta}^n, \vxi_i)}^2}
	\right]
	\\
	& \le 12 L \max_i (\tilde{q}_i(n)) \left[\tilde{\Lcal}^n(\vtheta^{n, k}) - \tilde{\Lcal}^n(\bar{\vtheta}^n)\right]
	+ 6  \sum_{ i = 1 }^M \tilde{q}_i^2(n)\E{\norm{ \nabla\Lcal_i(\bar{\vtheta}^n, \vxi_i))}^2}
	.
	\label{app:eq:S1_2}
	\end{align}
	
	Using the Lipschitz smoothness of the local loss functions, Assumption \ref{ass:smoothness} and Jensen inequality, we can bound $c(n, k)$ as
	\begin{align}
	c(n , k)
	\le  3 \E{\norm{\nabla \tilde{\Lcal}^n(\vtheta^{n, k}) - \nabla \tilde{\Lcal}^n(\bar{\vtheta}^n)}^2}
	\le 6 L \left[\tilde{\Lcal}^n(\vtheta^{n, k}) - \tilde{\Lcal}^n(\bar{\vtheta}^n)\right]
	\label{app:eq:S1_3}
	.
	\end{align}
	Substituting equation (\refeq{app:eq:S1_1}), equation (\refeq{app:eq:S1_2}), and equation (\refeq{app:eq:S1_3}) in equation (\refeq{app:eq:S1}),
	considering that 
	$ \max_i (\tilde{q}_i(n)) \le 1$, 
	 and summing over $K$
	 gives
	\begin{align}
	R(n)
	&\le 3  L^2 \phi(n)
	+ 18 L q^{-1}(n) Z(n)
	+ 6  \sigma_2(n) 
	\end{align}
	Using Lemma \ref{lem:phi_nk} to replace $\phi(n)$, and considering that $D \le 1 /2 < 1$ completes the proof.
	
\end{proof}

\begin{lemma}\label{lem:bound_Q}
	Under Assumption \ref{ass:smoothness} and \ref{ass:unbiased}, 
	considering that $\gamma_i(n) \le \beta q_i(n)$, 
	and considering $12 \rho^2 \left[ \alpha + \beta\right]  \tilde{\eta}^2 K^2 \tau^2 L^2
	\le 1/2$, we have
	\begin{align}
	\bar{Q}(N)
	& \le
	24 \rho \left[2 \alpha + \beta \right] \tilde{\eta}^2 K^2 L \bar{Z}(N)
	+ 6  \rho^2 \left[ \alpha D + 2 \beta \right] \tilde{\eta}^2 K^2 \Sigma_1(N)
	+ 12  \rho^2 \alpha \tilde{\eta}^2 K^2  \Sigma_2(N) 
	.
	\end{align}
	
\end{lemma}

\begin{proof}
	Considering the proof of Lemma \ref{lemma:one_SGD}, using the fact that $\gamma_i(n) \le \beta q_i(n)$, and Jensen inequality, we have	
	\begin{align}
	Q(n)
	&\le 
	q^2(n) \alpha  \tilde{\eta}^2 \E{\norm{\sum_{ i = 1 }^M \tilde{q}_i(n)\sum_{ k = 0 }^{K-1} g_i(\vtheta_i^{\rho_i(n), k}) }^2}
	+ q(n) \beta  \tilde{\eta}^2 \sum_{ i = 1 }^M \tilde{q}_i(n) \E{\norm{\sum_{ k = 0 }^{K-1} g_i(\vtheta_i^{\rho_i(n), k})}^2}
	\nonumber\\
	&\le 
	q^2(n) \alpha  \tilde{\eta}^2 K^2 R(n)
	+ q(n) \beta  \tilde{\eta}^2 K^2 S(n)
	\end{align}
	Using Lemma \ref{lem:bounding_R} to bound $R(n)$ and Lemma \ref{lem:phi_nk} to bound $S(n)$, we can thus bound $Q(n)$ with the previous global model distances to the optimum $Q(s)$, where $ \max(0, n - \tau) \le s \le n -1$, we thus have 	
	\begin{align}
	\frac{1}{\rho \tilde{\eta}^2 K^2}Q(n)
	& \le 12 \rho  \left[\alpha + \beta\right] \tau L^2 \sum_{ s = 1 }^\tau Q(n - s)
	+ 12 \left[2 \alpha + \beta \right] L Z(n)
	\nonumber\\
	& 
	+ 3 \rho  \left[ \alpha D + 2 \beta \right] \sigma_1(n)
	+ 6 \rho  \alpha \sigma_2(n) 
	.\label{app:eq:J1}
	\end{align}
	We can thus define $A(n)$ and $B(n)$ such that the bound of of equation (\refeq{app:eq:J1}) can be rewritten as in equation (\refeq{app:eq:rewriting}),  with its associated implications when taking the mean over $N$, reordering, and considering that $\tau A(n) \le 1/2$:
	\begin{align}
	Q(n)
	\le A(n) \sum_{ s = 1 }^\tau Q(n -s) 
	+ B(n)
	\Rightarrow
	\bar{Q}(N) 
	= \frac{1}{N}\sum_{ n = 0 }^{N-1} Q(n)
	\le 2 \frac{1}{N}  \sum_{ n = 0 }^{N-1} B(n)
	.
	\label{app:eq:rewriting}
	\end{align}
	Therefore, considering $12 \rho^2 \left[ \alpha + \beta\right]  \tilde{\eta}^2 K^2 \tau^2 L^2
	\le 1/2$ completes the proof.
	
	
\end{proof}

\subsection{Proof of Theorem \ref{theo:convergence_convex}}\label{app:subsec:proof_theo}
\begin{proof}
	Using Lemma \ref{lemma:one_SGD}, we have 
	\begin{align}
	\frac{1}{\tilde{\eta}}\Delta(n)
	&\le - 2  D(\vx, n)
	+ \rho^2 \alpha \tilde{\eta}  R(n)
	+  \rho  \beta \tilde{\eta} S(n)
	\end{align}
	Using Lemma \ref{lem:Dxn} to bound $D(\vx, n)$, Lemma \ref{lem:bounding_R} to bound $R(n)$, Lemma \ref{lem:phi_nk} to bound $S(n)$, and $3 \rho  \left[\alpha + \beta\right] \tilde{\eta} L \le 1$, we get
	\begin{align}
	\frac{1}{\tilde{\eta}}\Delta(n)
	&\le -2  \Xi(n)
	 + 8 \rho \tau L \sum_{s=1}^\tau  Q(n - s) 
	+ 4 D Z(n)
	+ 6 \rho \eta_l^2 (K-1)^2 L \sigma_1(n)
	\nonumber\\
	& 
	+ 12 \left[2 \alpha + \beta \right] \rho \tilde{\eta} L Z(n) 
	+ 3 \rho^2 \tilde{\eta} \left[ \alpha D + 2 \beta \right]  \sigma_1(n)
	+ 6  \rho^2 \alpha \tilde{\eta} \sigma_2(n) 
	\label{app:eq:C_12}
	.
\end{align}
	When considering the following intermediary result
	\begin{equation}
	\sum_{ n = 0 }^{N-1} K \Delta(n) \\
	= \E{\norm{\vtheta^{KN} - \vx}^2} - \norm{\vtheta^{0} - \vx}^2
	\ge - \norm{\vtheta^{0} - \vx}^2
	,
	\end{equation}
	reordering the terms, and taking the mean over $N$, we get
	\begin{align}
	2  \bar{\Xi}(N)
	&\le \frac{1}{ \tilde{\eta} KN }\E{ \norm{\vtheta^{0} - \vx}^2 }
	+ 8 \rho L \tau^2 \bar{Q}(N) 
	+ 4 D \bar{Z}(N)
	+ 6 \rho \eta_l^2 (K-1)^2 L \Sigma_1(N)
	\nonumber\\
	& 
	+ 12 \rho  \left[2 \alpha + \beta \right] \tilde{\eta} L \bar{Z}(N) 
	+ 3 \rho^2 \left[ \alpha D + 2 \beta \right] \tilde{\eta} \Sigma_1(N)
	+ 6  \rho^2 \alpha \tilde{\eta} \Sigma_2(N) 
	.
	\end{align}
	Using Lemma \ref{lem:bound_Q} to bound $\bar{Q}(N)$, and with $\nu = 16 \rho L$, we have
	\begin{align}
	2  \bar{\Xi}(N)
	&\le \frac{1}{\tilde{\eta} KN }\E{ \norm{\vtheta^{0} - \vx}^2 }
	+ 4 D \bar{Z}(N)
	+ 6 \rho \eta_l^2 (K-1)^2 L \Sigma_1(N)
	\nonumber\\
	& 
	+ 12 \rho \left[2 \alpha + \beta \right] \left[\tilde{\eta} + \nu \tilde{\eta}^2 K^2 \tau^2 \right] L \bar{Z}(N) 
	+ 3 \rho^2 \left[ \alpha D + 2 \beta \right] \left[\tilde{\eta} + \nu \tilde{\eta}^2 K^2 \tau^2 \right] \Sigma_1(N)
	\nonumber\\
	&+ 6  \rho^2 \alpha \left[\tilde{\eta} + \nu \tilde{\eta}^2 K^2 \tau^2 \right] \Sigma_2(N) 
	.
	\end{align}
	We note that when $ \bar{\Xi}(N) \le 0$, the claim follows directly. Therefore, we consider $ \bar{\Xi}(N) \ge 0$ for the rest of this proof. 
	We first note that 
	\begin{equation}
		\bar{Z}(N)
		= \bar{\Xi}(N)
		+ R(\{\Lcal^n\})
		\label{app:eq:decompo_loss}
		,
	\end{equation}
	and consider $\eta_l$ such that
	\begin{equation}
		2 
		- 4 D 
		- 12 \rho \left[2 \alpha + \beta \right] \left[\tilde{\eta} + \nu \tilde{\eta}^2 K^2 \tau^2 \right] L
		\ge 1 
		,
	\end{equation}
	which gives
	\begin{align}
	\bar{\Xi}(N)
	&\le \frac{1}{\tilde{\eta} KN}\E{ \norm{\vtheta^{0} - \vx}^2 }
	+ 4 D R(\{\Lcal^n\})
	+ 6 \rho \eta_l^2 (K-1)^2 L \Sigma_1(N)
	\nonumber\\
	& 
	+ 12 \rho \left[2 \alpha + \beta \right] \left[\tilde{\eta} + \nu \tilde{\eta}^2 K^2 \tau^2 \right] L R(\{\Lcal^n\})
	+ 3 \rho^2 \left[ \alpha D + 2 \beta \right] \left[\tilde{\eta} + \nu \tilde{\eta}^2 K^2 \tau^2 \right] \Sigma_1(N)
	\nonumber\\
	&+ 6  \rho^2 \alpha \left[\tilde{\eta} + \nu \tilde{\eta}^2 K^2 \tau^2 \right] \Sigma_2(N) 
	.
	\end{align}
	The 5th term can be simplified with the third one. Indeed, we consider a local learning rate such that $3 \rho^2 \tilde{\eta} L \le 1$, $48 \rho^3 \tilde{\eta}^2 K^2 \tau^2 L^2 \le1$, and we remind that $\alpha \le 1$. We thus have
	\begin{align}
	\bar{\Xi}(N)
	&\le \frac{1}{ \tilde{\eta} KN}\E{ \norm{\vtheta^{0} - \vx}^2 }
	+ \Ocal \left(\eta_l^2 (K-1)^2 \left[ R(\{\Lcal^n\}) + \Sigma_1(N) \right] \right) 
	\nonumber\\
	& 
	+ \Ocal \left( \alpha \left[\tilde{\eta} + \tilde{\eta}^2 K^2 \tau^2 \right] \left[ R(\{\Lcal^n\}) + \Sigma_2(N) \right] \right)
	\nonumber\\
	&+ \Ocal\left(\beta \left[\tilde{\eta} + \tilde{\eta}^2 K^2 \tau^2 \right] \left[ R(\{\Lcal^n\}) + \Sigma_1(N)\right]\right)
	.\label{app:eq:final_theo}
	\end{align}
	
	With 
	\begin{equation}
		\norm{\nabla \Lcal_i(\vtheta, \xi)}^2 
		\le 2 \norm{\nabla \Lcal_i(\vtheta, \xi) - \nabla \Lcal_i(\bar{\vtheta}, \xi)}^2
		+ 2 \norm{\nabla \Lcal_i(\bar{\vtheta}, \xi)}^2
		,
	\end{equation}
	we have 
	\begin{equation}
		\Sigma_2(N) 
		\le \max q_i(n) \Sigma_1(N) 
		\le \max q_i(n) \left[4L R({\Lcal^n}) + 2 \Sigma \right]
		.\label{app:eq:bound_Sigma}
	\end{equation}	
	Finally, substituting equation (\refeq{app:eq:decompo_loss}) and (\refeq{app:eq:bound_Sigma}) in equation (\refeq{app:eq:final_theo}) completes the proof.

\end{proof}

\subsection{Simplifying the constraint on the learning rate}\label{app:subsec:learning_rate}
The constraints on the learning rate can be summarized as 
$D = 6 \eta_l^2 (K-1)^2 L^2 \le 1/2$ (Lemma \ref{lem:phi_nk}),
$12 \rho^2 \left[ \alpha + \beta\right]  \tilde{\eta}^2 K^2 \tau^2 L^2
\le 1/2$ (Lemma \ref{lem:bound_Q}),
$3 \rho  \left[\alpha + \beta\right] \tilde{\eta} L \le 1$ (Theorem \ref{theo:convergence_convex}),
$2 
- 4 D 
- 12 \rho \left[2 \alpha + \beta \right] \left[\tilde{\eta} + \nu \tilde{\eta}^2 K^2 \tau^2 \right] L
\ge 1$ (Theorem \ref{theo:convergence_convex}),
$3 \rho^2 \tilde{\eta} L \le 1$ (Theorem \ref{theo:convergence_convex}), 
and $48 \rho^3 \tilde{\eta}^2 K^2 \tau^2 L^2 \le1$ (Theorem \ref{theo:convergence_convex}).

We note that $\alpha \le 1$, and $\beta \le 1$. We thus propose the following sufficient conditions to satisfy the conditions above
\begin{align}
	48 \eta_l^2 (K-1)^2 L^2 \le 1,
	144 \rho^2 \tilde{\eta} L \le 1,
	\text{ and }
	2304 \rho^3 \tilde{\eta}^2 K^2 \tau^2 L^2 \le 1
	,
\end{align}
which can further be simplified with
\begin{equation}
	\eta_l
	\le \frac{1}{48 K L }
	\min \left(1, \frac{1}{3 \rho^{2} \eta_g (\tau+1)}\right)
	.
\end{equation}
\section{Proof of Theorem \ref{theo:bias_convex}}\label{app:sec:theo_bias}

In this proof, we consider $\tilde{\Lcal}^n = q^{-1}(n) \Lcal^n$.

\subsection{Useful Lemma}
\begin{lemma} \label{lem:async_FL_bias}
	The difference between the gradients of $\Lcal(\vtheta)$ and $\Lcal(\vtheta)$ can be bounded as follow
	\begin{align}
	\norm{\nabla \Lcal(\vtheta) - \nabla \tilde{\Lcal}^n (\vtheta)}^2
	 \le 4L \chi_n^2 [ \tilde{\Lcal}^n(\vtheta) -\sum_{j \in W_n} \tilde{s}_j(n) \Lcal_j(\vtheta_j^*)] 
	+ 4L \sum_{j \notin W_n} r_j  [ \Lcal_j(\vtheta) - \Lcal_j(\vtheta_j^*)] ,
	\end{align}
	where $W_n = \{j : s_j(n) > 0\}$ and $\chi_n^2 = \sum_{j \in W_n } (r_j - \tilde{s}_j(n))^2/\tilde{s}_j(n)$.
\end{lemma}

\begin{proof}
	We have $\sum_{ j = 1 }^J s_j(n) = \sum_{ i = 1 }^M q_i(n) = q(n)$.
	Hence, by definition of $\Lcal(\vtheta)$ and $\Lcal^n(\vtheta)$, we have
	\begin{align}
		\nabla \Lcal(\vtheta) - \nabla \tilde{\Lcal}^n(\vtheta)
		& = \sum_{j=1}^{J} (r_j - \tilde{s}_j(n) ) \nabla \Lcal_j(\vtheta)\\
		& = \sum_{j \in W_n} \frac{r_j - \tilde{s}_j(n)}{\sqrt{\tilde{s}_j(n)}} \sqrt{\tilde{s}_j(n)}\nabla \Lcal_j(\vtheta)
		+ \sum_{j \notin W_n} r_j \nabla \Lcal_j(\vtheta)
		.
	\end{align}
	Applying Jensen and Cauchy-Schwartz inequality gives
	\begin{align}
		\norm{\nabla \Lcal(\vtheta) - \nabla \tilde{\Lcal}^n (\vtheta)}^2
		& \le 2 \norm{ \sum_{j \in W_n} \frac{r_j - \tilde{s}_j(n)}{\sqrt{\tilde{s}_j(n)}} \sqrt{\tilde{s}_j(n)} \nabla \Lcal_j(\vtheta) }^2 
		+ 2 \norm{ \sum_{j \notin W_n} r_j \nabla \Lcal_j(\vtheta)}^2 
		\\
		& \le 2 \left[\sum_{j \in W_n} \frac{(r_j - \tilde{s}_j(n))^2}{\tilde{s}_j(n)}  \right] 
		\sum_{j = 1}^J \tilde{s}_j(n) \norm{ \nabla \Lcal_j(\vtheta) }^2 
		\nonumber\\
		& + 2  \left[\sum_{j \notin W_n}  r_j\right] \sum_{j \notin W_n} r_j  \norm{ \nabla \Lcal_j(\vtheta) }^2 
	\end{align}
	Considering the Lipschitz smoothness of the clients loss function, and $\sum_{ j \notin W_n} r_j \le  1$ completes the proof.
	
	
\end{proof}

\subsection{Proof of Theorem \ref{theo:bias_convex}}
\begin{proof}
	Using Jensen inequality and Lemma \ref{lem:async_FL_bias} gives
	\begin{align}
	\norm{\nabla \Lcal (\vtheta)}^2
	& \le 2 \norm{\nabla \Lcal ( \vtheta) - \frac{1}{q(n)}\nabla \Lcal^n( \vtheta)}^2 
	+ 2 \norm{\frac{1}{q(n)} \nabla \Lcal^n (\vtheta)}^2\\
	& \le 4L \left[\chi_n^2 \frac{1}{q(n)} + \frac{1}{q^2(n)} \right] [ \Lcal^n(\vtheta) - \Lcal^n(\bar{\vtheta}^n)] \nonumber\\
	& + \chi_n^2 \frac{1}{q(n)}4L [ \Lcal^n(\bar{\vtheta}^n) -\sum_{j \in W_n} s_j(n) \Lcal_j(\vtheta_j^*)] \nonumber\\
	&+ 4L  \sum_{j \notin W_n} r_j  [ \Lcal_j(\vtheta) - \Lcal_j(\vtheta_j^*)]
	\end{align}
	We take the maximum of $\chi_n^2$ and $q(n)$, the mean over the $KN$ serial SGD steps, and use Theorem \ref{theo:convergence_convex} to complete the proof
	.

\end{proof}

\section{Applying Theorem \ref{theo:relaxed_sufficient_conditions}}
\label{app:sec:applying}

This section extends Section \ref{sec:applications}, where we apply Theorem \ref{theo:relaxed_sufficient_conditions} to centralized learning (Section \ref{subsec:centralized_learning}) and synchronous \textsc{FedAvg} with unbiased and biased client sampling (Section \ref{subsec:client_sampling_unbiased} and \ref{subsec:client_sampling_biased} respectively).

\subsection{Centralized Learning}\label{subsec:centralized_learning}

In this setting, one client, i.e. $M = 1$, learns a predictive model on its own data. In this case, we always have $\tilde{q}_1(n) =1$, and the resulting optimization problem is always proportional to $\Lcal = \Lcal_1$ which thus gives $R(\{\Lcal^n\}) \le R(\Lcal) = 0$. 
There is no gradient delay ($\tau =1$), while the clients always participate at each optimization round ($\alpha = 1$ and $\beta=0$), while the global learning rate is redundant with the local learning rate ($\eta_g =1$). The server performs $KN$ SGD steps. All these considered elements give
\begin{equation}
\label{eq:epsilon_centra}
\epsilon
= \Ocal\left(\frac{\norm{\vtheta^0 - \vtheta^*}^2}{\eta_l K N}\right)
+ \Ocal \left(\eta_l \EE{\vxi}{\norm{\nabla \Lcal(\vx, \vxi)}^2} \right)
.
\end{equation}
With equation (\refeq{eq:epsilon_centra}), we retrieve standard convergence guarantees for centralized ML derived in \cite{SGD_review}.

\subsection{Unbiased client sampling ($q_i(n) = p_i$)}\label{subsec:client_sampling_unbiased}
We define by $S_n$ the set of sampled clients performing their local work at optimization step $n$. Setting $\Delta t^n = \max_{i \in S_n } T_i$, with $T_i = \infty$ for the clients that are not sampled, and thus not in $S_n$, gives $\mathbb{P}( T_i \le \Delta t^n) = \mathbb{P}( i \in S_n )$. $S_n$ is independent from the clients hardware capabilities and is decided by the server. This allows to pre-compute $\mathbb{P}( T_i \le \Delta t^n)$ and to allocate to each client the aggregation weight $d_i$ such that $q_i = p_i$. 

Standard unbiased client sampling schemes include sampling $m$ clients uniformly without replacement \citep{OnTheConvergence} or sampling $m$ clients according to a Multinomial distribution \citep{FedProx}.
\cite{OnTheImpact} shows that both Uniform and MD sampling satisfy Assumption \ref{ass:clients_covariance}. 
In particular, in those setting, the term $\alpha\le 1$ is proportional to $m$, the amount of sampled clients, while $1 \ge \beta>0$ is inversely proportional to $m$. 
We get
\begin{align}
\epsilon
&= \Ocal\left(\frac{1}{\eta_g \eta_l K N}\right)
+ \Ocal \left(\eta_g \eta_l \alpha \frac{1}{M}\Sigma \right)
+ \Ocal \left(\eta_l^2 (K-1)^2 \Sigma \right)
+ \Ocal \left(\eta_g \eta_l \beta \Sigma \right)
\label{eq:epsilon_sync_unbiased}
.
\end{align}
The second term, proportional to $\alpha/M$, is reduced at the expense of the introduction of a fourth term proportional to $\beta$.
In turn, it still provides faster optimization rounds with $\Delta t^n = \max_{i \in S_n } T_i$ and $N = \Ocal \left(T/ \E{\max_{i \in S_n } T_i}\right)$. 
FedAvg with client sampling generalizes FedAvg with full client participation ($\alpha = 1$ and $\beta=0$).

\subsection{Biased client sampling ($q_i(n) \neq p_i$)} \label{subsec:client_sampling_biased}
The condition $q_i(n) = p_i$ imposes the design of new client sampling based on the clients data heterogeneity. Nevertheless, we show convergence of biased client samplings where $m$ clients are selected according to a deterministic criterion,  e.g. when selecting the $m$ clients with the highest loss \citep{PowerOfChoice}, or when selecting the $m$ clients with the most available computation resources \citep{sampling_mobile_edge}. 
In this case, $\mathbb{P}( i \in S_n ) = 0/1$, with 1 if a client satisfies the criterion and 0 otherwise. In this case, no weighting scheme can make an optimization round unbiased.
We also have $\mathbb{P}( \{i, j\} \in S_n ) = \mathbb{P}( i \in S_n ) \mathbb{P}( j \in S_n )$, which gives $\alpha = 1$ with $\beta=0$. 
Without modification, this client sampling cannot satisfy the relaxed sufficient conditions of Theorem \ref{theo:relaxed_sufficient_conditions} and thus converges to a suboptimum point. 
This drawback can be mitigated by allocating a part of time in the window $W$ to sample clients according to the criterion, and the rest of the window to consider clients such that $q_i = p_i$ is satisfied over $W$ optimization rounds. 
By denoting $\epsilon_{\textsc{FedAvg}}$ the convergence guarantees (\refeq{eq:epsilon_sync_FL}), we have
\begin{equation}
\epsilon
= \epsilon_{\textsc{FedAvg}}
+ \Ocal \left(\eta_g\eta_l (W - 1) K\right)
\label{eq:epsilon_sync_bias}
.
\end{equation}
We note that equation (\refeq{eq:epsilon_sync_bias}) provides a looser bound than equation (\refeq{eq:epsilon_sync_FL}) in term of optimization rounds $N$.
Still, this bound is informative and shows that, with minor changes, biased clients sampling based on a deterministic criterion can be proven to converge to the FL optimum.

\section{Additional Experiments}\label{app:sec:experiments}

\begin{figure}
	\centering
	\includegraphics[width=\linewidth]{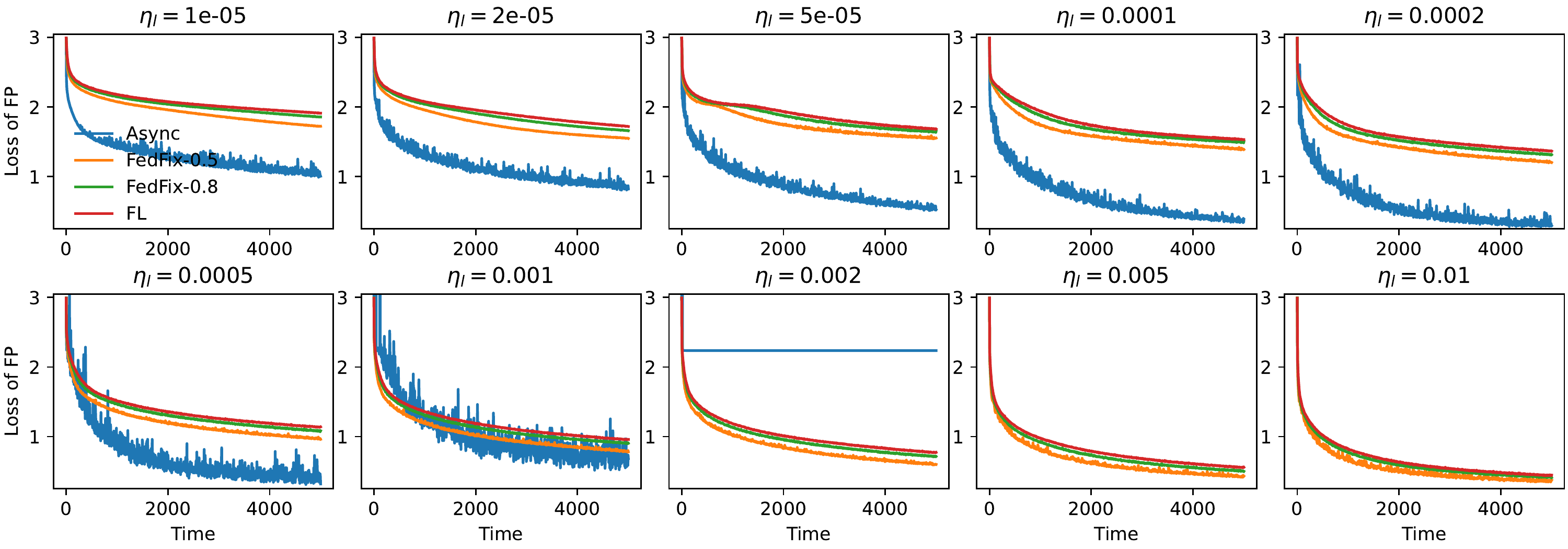}
	\caption{Evolution of federated problem (\ref{eq:original_problem_general}) loss for CIFAR10 and time scenario $F80$ with $M=20$ and $K=10$.}
	\label{app:fig:explore_CIFAR10_20}
\end{figure}

\begin{figure}
	\centering
	\includegraphics[width=\linewidth]{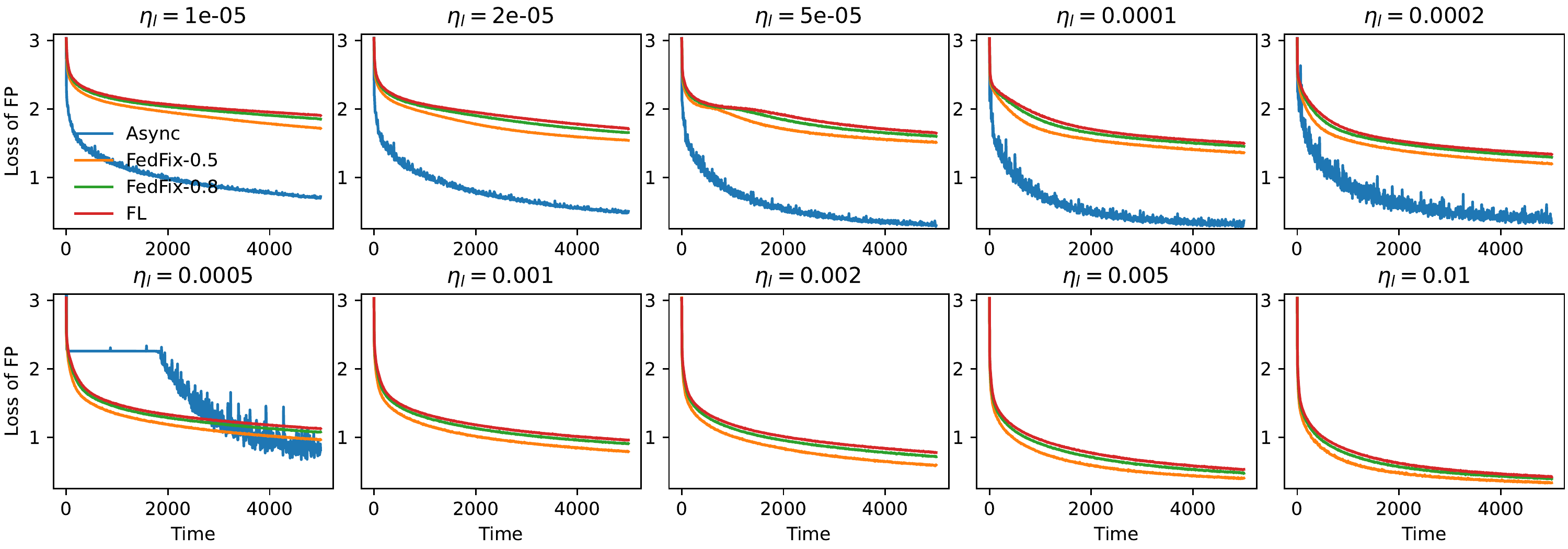}
	\caption{Evolution of federated problem (\ref{eq:original_problem_general}) loss for CIFAR10 and time scenario $F80$ with $M=50$ and $K = 10$.}
	\label{app:fig:explore_CIFAR10_50}
\end{figure}

\begin{figure}
	\centering
	\includegraphics[width=\linewidth]{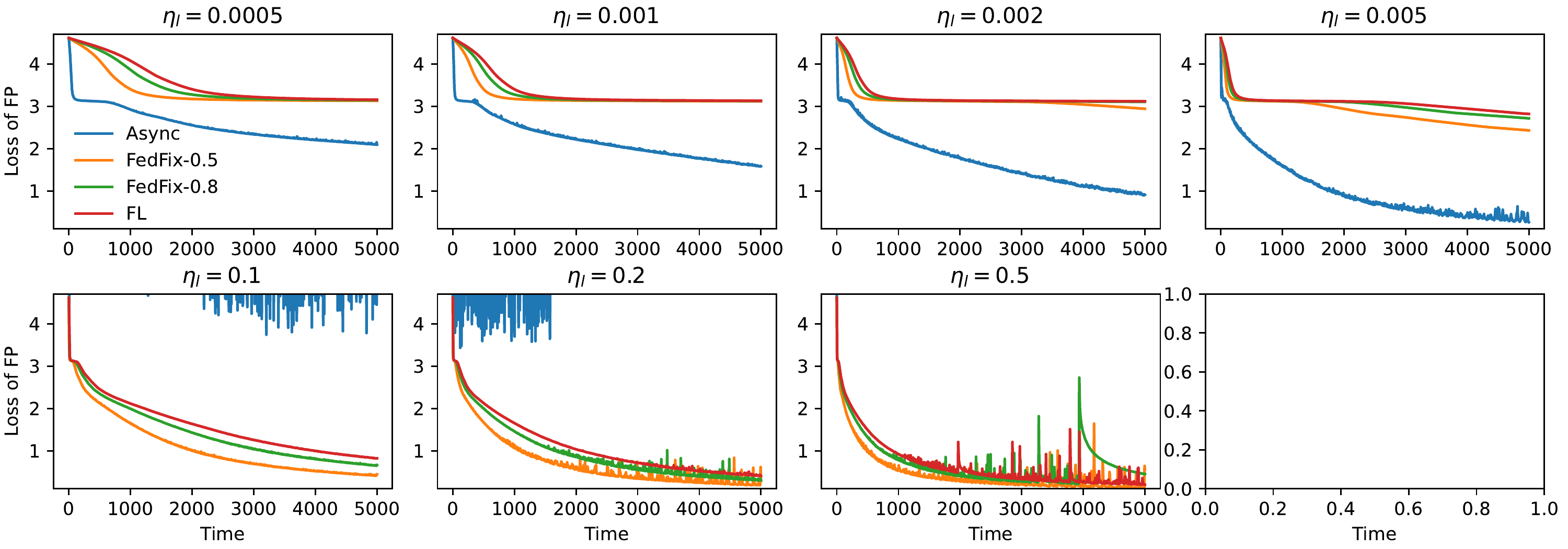}
	\caption{Evolution of federated problem (\ref{eq:original_problem_general}) loss for Shakespeare and time scenario $F80$ with $M=20$ and $K = 10$.}
	\label{app:fig:explore_Shakespeare_20}
\end{figure}





\subsection*{Acknowledgments}

This work has been supported by the French government, through the 3IA C\^ote d'Azur Investments in the Future project managed by the National Research Agency (ANR) with the reference number ANR-19-P3IA-0002, and by the ANR JCJC project Fed-BioMed 19-CE45-0006-01. The project was also supported by Accenture.
The authors are grateful to the OPAL infrastructure from Université C\^ote d'Azur for providing resources and support.
\vskip 0.2in
\bibliographystyle{apalike}
\bibliography{biblio}

\end{document}